\documentclass[twoside,11pt]{article}

\usepackage{thm-restate}

\usepackage[preprint]{jmlr2e}

\usepackage{hyperref}       
\usepackage{url}            
\usepackage{booktabs}       
\usepackage{amsfonts}       
\usepackage{nicefrac}       
\usepackage{microtype}      
\usepackage{xcolor}         
\usepackage{amsmath}

\usepackage{amsmath}
\usepackage{amssymb}
\usepackage{mathtools}
\usepackage{mathrsfs}

\usepackage{array}
\usepackage{threeparttable}

\usepackage{enumerate}

\usepackage{enumitem}

\usepackage{wrapfig}

\usepackage{dsfont}
\usepackage{comment}
\usepackage{xspace}

\usepackage{cleveref}

\usepackage{threeparttable}

\Crefname{assumption}{Assumption}{Assumptions}

\newcommand{\eg}{e.g.\xspace}
\newcommand{\as}{a.s.\xspace}

\newcommand{\wrt}{{\textit{ w.r.t.}}\xspace}

\newcommand{\ie}{{\textit{ i.e.}}\xspace}

\newcommand{\R}{\mathds{R}}
\newcommand{\Q}{\mathds{Q}}
\newcommand{\risk}{\mathcal{R}}
\newcommand{\er}{\widehat{\mathcal{R}}_S}

\newcommand{\el}{\widehat{L}_S}

\newcommand{\zcal}{\mathcal{Z}}

\newcommand{\ocal}{\mathcal{O}}

\newcommand{\Rd}{{\R^d}}

\newcommand{\Eof}[2][]{\mathds{E}_{#1} \left[ #2 \right]}

\newcommand{\Pof}[2][]{\mathds{P}_{#1} \left( #2 \right)}

\newcommand{\prob}{\mathds{P}}
\newcommand{\normof}[1]{\left\Vert #1 \right\Vert}
\newcommand{\klb}[2]{\text{\normalfont{{KL}}}\left(#1 || #2 \right)}

\newcommand{\entphi}[2]{\mathrm{Ent}_{#1}^\Phi \left(#2\right)}
\newcommand{\entphibar}[2]{\overline{\mathrm{Ent}}_{#1}^\Phi \left(#2\right)}
\newcommand{\entbar}{\overline{\mathrm{Ent}}}
\newcommand{\setof}[1]{\left\{ #1 \right\} }
\newcommand{\datadist}{\mu_z^{\otimes n}}

\newcommand{\fcal}{\mathcal{F}}

\newcommand{\der}{\text{\normalfont d}}

\newcommand{\intrd}{\int_{\Rd}}

\newcommand{\timeder}{\frac{\der }{\der t}}

\newcommand{\by}[1]{\quad\text{(#1)}}

\newcommand{\tv}{\mathrm{TV}}

\newcommand{\ecal}{\mathcal{E}}
\newcommand{\bcal}{\mathcal{B}}
\newcommand{\borel}{\mathcal{B}(\mathds{R}^d)}
\newcommand{\pcal}{\mathcal{P}}
\newcommand{\probameasures}{\mathcal{P}(\Rd)}
\newcommand{\var}{\mathrm{Var}}
\newcommand{\N}{{\mathds{N}}}
\newcommand{\ent}[2][\pi]{\mathrm{Ent}_{#1}\left( #2 \right)}
\newcommand{\entrm}{\mathrm{Ent}}
\newcommand{\cd}{\mathrm{CD}}

\newcommand{\lcal}{{\mathcal{L}}}

\newcommand{\law}{{\mathrm{Law}}}
\newcommand{\trace}{{\mathrm{Tr}}}
\newcommand{\vbar}{{\overline{v}}}
\newcommand{\Xbar}{{\overline{X}}}
\newcommand{\Ybar}{{\overline{Y}}}
\newcommand{\rhobar}{{\overline{\rho}}}
\newcommand{\Pbar}{{\overline{P}}}
\newcommand{\sigmabar}{{\overline{\sigma}}}
\newcommand{\Lrm}{\mathrm{L}}
\newcommand{\Irm}{\mathrm{I}}
\newcommand{\EofLigne}[2][]{\mathbb{E}_{#1} [ #2 ]}
\newcommand{\normofLigne}[1]{\Vert #1 \Vert}
\newcommand{\fisher}[2]{\mathscr{I}\left(#1 | #2 \right)}
\newcommand{\Wrm}{\mathrm{W}}

\newtheorem{assumption}{Assumption}

\usepackage{lastpage}
\jmlrheading{23}{2022}{1-\pageref{LastPage}}{1/21; Revised 5/22}{9/22}{21-0000}{Author One and Author Two}

\ShortHeadings{Generalization Bounds for Markov Algorithms}{Dupuis et al.}
\firstpageno{1}

\begin{document}

\title{Generalization Bounds for Markov Algorithms through Entropy Flow Computations}


\author{\name Benjamin Dupuis* \email benjamin.dupuis@inria.fr\\
 \addr INRIA - Département d’Informatique de l’Ecole Normale Supérieure\\
 \addr PSL Research University - CNRS\\
 \addr Paris, France
 \AND
 \name Maxime Haddouche
 \email maxime.haddouche@inria.fr\\
 \addr INRIA - Département d’Informatique de l’Ecole Normale Supérieure \\
 \addr PSL Research University - CNRS\\
 \addr Paris, France
 \AND
 \name George Deligiannidis \email george.deligiannidis@stats.ox.ac.uk\\
 \addr Department of Statistics, University of Oxford \\
 \addr Oxford, UK
 \AND
 \name Umut Simsekli \email umut.simsekli@inria.fr\\
 \addr INRIA - Département d’Informatique de l’Ecole Normale Supérieure\\
 \addr PSL Research University - CNRS\\
 \addr Paris, France
 \\
  \\
 {\normalfont \textbf{*}} Corresponding author.
}

\editor{My editor}

\maketitle

\begin{abstract}
    Many learning algorithms can be represented as Markov processes, and understanding their generalization error is a central topic in learning theory.
    For specific continuous-time noisy algorithms, a prominent analysis technique relies on information-theoretic tools and the so-called ``entropy flow'' method.
    This technique is compatible with a broad range of assumptions and leverages the convergence properties of learning dynamics to produce meaningful generalization bounds, which can also be informative or extend to discrete-time settings.
    Despite their success, existing entropy flow formulations are limited to specific noise and algorithm structures (\eg, Langevin dynamics).
    In this work, we exploit new technical tools to extend its applicability to all learning algorithms whose iterative dynamics is governed by a time-homogeneous Markov process. 
    Our approach builds on a principled continuous-time approximation of Markov algorithms and introduces a new, exact entropy flow formula for such processes.
    Within this unified framework, we establish novel connections to a well-studied family of modified logarithmic Sobolev inequalities, which we use to connect the generalization error to the ergodic properties of Markov processes.
    Finally, we provide a detailed analysis of all the terms appearing in our theory and demonstrate its effectiveness by deriving new generalization bounds for several concrete algorithms.
\end{abstract}


\section{Introduction}
\label{sec:introduction}

Understanding the generalization error of machine learning algorithms remains a crucial challenge.
We model such learning problems by a tuple $(\ell, \zcal, \mu_z, \mathcal{H})$ where $\mathcal{H}$ is a parameter space ($\mathcal{H}=\Rd$ in our study), $\zcal$ is a data space endowed with a $\sigma$-algebra $\fcal$, $\mu_z$ is a data distribution, and $\ell : \mathcal{H} \times \zcal \to \R_+$ is a loss function. 
We aim to minimize the \emph{population risk},
$$\risk(w) := \int \ell(w,z) \der \mu_z(z) = \Eof[Z \sim \mu_z]{\ell(w, Z)}, \quad w \in \Rd,$$
over the parameter space $\mathcal{H} = \Rd$. Unfortunately, since $\mu_z$ is unknown, practitioners resort to the minimization of the \emph{empirical risk}, 
$$\er(w) = \frac1{n} \sum_{i=1}^n \ell(w,z_i), \quad w \in \Rd,$$
where $S := (z_1,\dots,z_n)\sim\datadist$ is a dataset sampled from $\mu_z$.

Modern machine learning systems achieve this minimization through the use of \emph{stochastic optimization algorithms}. In our study, we focus on iterative algorithms with a Markov chain structure: $X_{k+1}^S = F(X_k^S, U_{k}, S)$ where $S\in\zcal^n$ and $U_k$ denotes the internal randomness of the algorithm at step $k$, independent of $S$. This encompasses many popular algorithms, including stochastic gradient descent (SGD) with constant step size \citep{dieuleveut_bridging_2018} and stochastic gradient Langevin dynamics (SGLD). 
To assess the quality of learning beyond $S$, it is classical to provide \emph{generalization bounds} on the learned parameter $X_k^S$, \ie, upper bounds on the quantity $G_S(X_k^S)$ where $G_S(w) := \risk(w) - \er(w)$. 
To provide computable guarantees, a popular approach is to derive high probability bounds of the form\footnote{We use $\lesssim$ for informal statements omitting absolute constants or weakly relevant terms.},
\begin{align}
    \label{eq:informal-high-probability-bound}
     \Pof[S\sim\datadist]{\Eof{G_S(X_k^S)|S} \lesssim \sqrt{\frac{\texttt{Complexity} + \log(1/\zeta)}{n}}} \geq 1 - \zeta,
\end{align}
where \texttt{Complexity} translates a certain facet of the problem, for instance, the Rademacher complexity \citep{bartlett2002rademacher} or the VC dimension \citep{vapnik2000learning}.

\paragraph{Generalization bounds for iterative algorithms.} 
As Rademacher complexity or VC dimension are algorithm-independent, they are not tailored to exploit the properties of the learning algorithm.
Hence, alternative approaches such as algorithmic stability \citep{bousquet_stability_2002} emerged, which yielded generalization bounds for SGD (among others) \citep{hardt_train_2016,feldman_high_2019} exploiting, in particular, its Markovian properties \citep{zhu_uniform--time_2023-1}. Unfortunately, these bounds often rely on relatively strong assumptions (convexity, Lipschitz, and/or Lipschitz gradient) and might not be time-uniform in non-convex settings \citep{bassily_stability_2020}.  

Another prospect is that of information-theoretic bounds, which provided \emph{expected} bounds for noisy algorithms (\eg, SGLD) \citep{xu_information-theoretic_2017,negrea_information-theoretic_2020,haghifam_sharpened_2020}. These methods have also been extended to SGD by \citet{neu_information-theoretic_2021} at the expense of potential time- and dimension-dependence, and by \citet{clerico_generalisation_2023} in the case of gradient-Lipschitz losses.
Of particular interest to us are the PAC-Bayesian bounds \citep{mcallester_pac-bayesian_1999,catoni_pac-bayesian_2007} where the term \texttt{Complexity} in \Cref{eq:informal-high-probability-bound} is typically expressed as $\klb{\law(X_k^S)}{\pi}$, where $\pi$ is a data-free `prior' distribution, $\law(X_k^S)$ is the `posterior' distribution, and $\klb{\cdot}{\cdot}$ is the Kullback-Leibler divergence (KL).
More recently, \citet{harel2025temperatureneedgeneralizationlangevin} obtained time-uniform PAC-Bayesian bounds for any Markov algorithm that involves the Gibbs potential of the invariant distribution of the algorithm. However, the existence of this invariant distribution is unclear in many applications, and the Gibbs potential might be intractable in practice. 
The Markov structure of SGD was also exploited by \citet{camuto_fractal_2021,hodgkinson2022generalizationboundsusinglower} through geometric properties, at the cost of non-explicit mutual information terms.

\paragraph{The particular case of continuous-time algorithms.} In parallel to these developments, many studies have focused on `continuous-time algorithms' $(Y_t^S)_{t \geq 0}$, typically represented by stochastic differential equations (SDE), for their structure is often easier to understand. Such procedures are often seen as an approximation of discrete-time methods.
A fundamental example is SGD, approximated by Langevin processes \citep{cheng_stochastic_2020,mandt_variational_2016,anastasiou2019normalapproximationstochasticgradient,xie2021diffusiontheorydeeplearning} and heavy-tailed SDEs \citep{simsekli_tail-index_2019,gurbuzbalaban2020heavy,raj2023algorithmic2}.
Continuous-time analysis can also be used as an intermediary to reach bounds on discrete-time methods by using interpolation techniques \citep{mou_generalization_2017}, which unfortunately only hold for specific noise structures.
Moreover, the approximation of discrete-time optimizers by continuous-time dynamics remains largely disputed \citep{li_validity_2021,wojtowytsch2021stochasticgradientdescentnoise} and restricted to specific settings, such as small learning rates \citep{li_stochastic_2018} or high-dimensional limits \citep{ben-arous_high_dimensional_2022}.

\paragraph{A crucial toolbox: the entropy flow method.} Among continuous-time algorithms, the generalization error of continuous Langevin dynamics (CLD) (and its discrete-time counterpart SGLD) has been studied through many of the techniques mentioned above \citep{raginsky_non-convex_2017,farghly_time-independent_2021,dupuis_uniform_2024}, notably involving PAC-Bayesian theory and information-theoretic tools \citep{mou_generalization_2017,li_generalization_2020,futami_time-independent_2023} where the goal is to upper bound the KL divergence $\klb{\law(Y_T^S)}{\pi_t}$, with $\pi_t$ a (possibly time-dependent) prior.
These techniques often rely on the so-called \emph{entropy flow}, which informally consists in the following derivations, up to absolute constants,
\begin{align}
    \label{eq:informal-entropy-flow-sgld}
    \frac{\der}{\der t} \klb{\law(Y_t^S)}{\pi_t} =  \mathrm{A} - \fisher{\rho_t^S}{\pi}  \leq \mathrm{B}  - \gamma \klb{\law(Y_t^S)}{\pi_t},
\end{align}
where $\fisher{\rho_t^S}{\pi}$ is a relative Fisher information term \citep{hyvarinen_estimation_2005}, and $\mathrm{A}$ and $\mathrm{B}$ are quantities that are usually dependent on the stochastic gradient, $\law(Y_t^S)$, and $\pi_t$ (see \citep{mou_generalization_2017}). The above inequality is a consequence of the celebrated logarithmic Sobolev inequality (LSI) \citep{gross_logarithmic_1975-1}, which must be satisfied by $\pi_t$.
In most cases, the prior is associated with a data-independent ergodic Markov process (\eg, its invariant measure), which is the source of the inequality in \Cref{eq:informal-entropy-flow-sgld}.
These approaches reveal a fundamental trade-off on the prior between having ergodic properties and being `close enough' to the posterior distribution $\law(Y_t^S)$ to make the term $\mathrm{B}$ small, see \citep{mou_generalization_2017} for more details.
Once \Cref{eq:informal-entropy-flow-sgld} is stated, a combination of information-theoretic bounds and Grönwall's lemma \citep{Gronwall1919NoteOT} leads to informative generalization bounds, which can be shown to be time uniform under additional assumptions (\eg, Lipschitz losses).

Therefore, this technique has three main advantages compared to the stability-based and other information-theoretic approaches mentioned above: \textit{(i)} it has the ability to produce time-uniform generalization bounds, \textit{(ii)} it rely on a flexible set of assumptions, and \textit{(iii)} it explicitly exploits intrinsic convergence properties of the algorithm through the use of LSIs.
Despite recent extensions to differential privacy \citep{chourasia_differential_2022} and $\alpha$-stable noise \citep{dupuis2025renyidifferentialprivacyheavytailed,dupuis_generalization_2024}, the entropy flow method remains limited to specific noise structures (\ie, Gaussian or $\alpha$-stable), as it requires a precise description of the time evolution of the density of the Markov process (\eg, a Fokker-Planck equation), which is a core element of all the aforementioned works.



\subsection{Contributions}
\label{sec:contributions}

In this work, we propose alleviating the issues mentioned above by extending the scope of the entropy flow method to all Markov algorithms.
At the heart of our approach lies a new class of continuous-time approximations of general (discrete-time) Markov algorithms with formal guarantees.
For a given Markov algorithm $X_{k+1}^S = F(X_k^S, U_k, S)$, we define the \emph{Poissonization} of $(X_k^S)_{k\in\N}$ as the continuous-time process $Y_t^S := X^S_{N_t}$, where $N_t$ is a Poisson process \citep{lasota_chaos_1994} (see \Cref{def:poisson-process}). 
This technique has been classically used in the analysis of the convergence of Markov chains \citep{diaconis_logarithmic_1996,chen_logarithmic_2008,caputo_entropy_2024,del_moral_contraction_2003,wang_transport-information_2020} and has recently emerged in optimization theory in \cite{even_continuized_2021} to study Nesterov acceleration.

To obtain generalization bounds, we derive an exact and compact entropy flow formula for Poissonized Markov algorithms.
Rather than a simple extension, this formula relies on new technical tools. In particular, the role of the Fokker-Planck equation is now replaced by a general ``Boltzmann equation''. 
We connect our theory to a class of `modified' LSIs \citep{diaconis_logarithmic_1996} and show that this can lead to generalization bounds with improved time dependence compared to classical methods. 
This effectively connects the generalization error of Markov algorithms to ergodic theory.
Finally, we provide a rigorous and easy-to-use analysis of all the terms appearing in our bounds and apply our theory to obtain generalization bounds for several practical learning algorithms.
This shows that, similar to the classical entropy flow method, our framework is compatible with a wide range of assumptions, but is applicable to a broader class of algorithms.

Our detailed contributions are listed below.

\begin{enumerate}[noitemsep,nosep]
    \item We analyze the generalization error of Markov algorithms through their Poissonization. We first show that Poissonization is a sound continuous-time proxy and then, with the notation above, prove high-probability bounds on the \emph{Poissonized} generalization error, 
    $$
        \Eof{G_S(Y_t^S) | S} = e^{-t}\sum_{k \in \N} \frac{t^k}{k!} \Eof{G_S(X_k^S) | S}.
    $$
    \item To achieve this goal, we derive a closed-form expression for the entropy flow of Poissonized algorithms. For a prior distribution $\pi$, our formula is informally given by
    $$
        \timeder \klb{\law(Y_t^S)}{\pi} = \Delta_S(t) - \ecal_\pi (t),
    $$
    where $\Delta_S(t)$ is a `distance' between two Markov kernels: {\it (i)} the kernel of the process $(X_k^S)_{k\in\N}$ and {\it (ii)} a kernel $P$ whose $\pi$ is the invariant measure; the term $\ecal_\pi (t)$ is called the \emph{Dirichlet form} and characterizes the convergence properties of $P$ and $\pi$.
    \item By making the connection with PAC-Bayesian bounds and a class of modified LSIs, we derive generic generalization bounds of the form,
    \begin{align}
        \label{eq:informal-generic-bound}
       \Pof[S \sim\datadist]{ \Eof{G_S(Y_T^S) | S} \lesssim \sqrt{\frac1{n} \int_0^T e^{-\gamma (T - t)} \Delta_S(t) \der t + \log\frac{3}{\zeta}}} \geq 1 - \zeta,
     \end{align}
     where $\gamma$ is a modified LSI constant associated with the prior distribution.
     \item We present various techniques to analyze the term $\Delta_S(t)$ in several settings for noisy and non-noisy algorithms. 
     \item Finally, we apply our framework in several examples. First, we recover Poissonized counterparts of classical generalization bounds for SGLD. Then, we use our entropy flow method to obtain new generalization bounds for both SGD and a recently proposed gradient descent method with noise injection \citep{orvieto_anticorrelated_2023}. 
\end{enumerate}

\paragraph{Organization of the paper.}
We first recall some technical background in \Cref{sec:technical-background}. Our framework is presented in \Cref{sec:poissonization-entropy-flow}, where we also derive generic generalization bounds. \Cref{sec:controlling-discrepancy} is dedicated to understanding the term $\Delta_S(t)$ appearing in \Cref{eq:informal-generic-bound} in several settings. Finally, we detail the application of our framework to concrete Markov algorithms in \Cref{sec:applications}. All omitted proofs can be found in the appendix.

\paragraph{Notation.} 
Let $\borel$ and $\probameasures$ denote the Borel sets and the Borel probability measures on $\Rd$. The absolute continuity of $\mu \in \probameasures$ \wrt $\nu \in \probameasures$ is written as $\mu \ll \nu$. The Kullback-Leibler (KL) divergence is defined as $\klb{\mu}{\nu} := \int \log (\der \mu / \der \nu) \der \mu$ when $\mu\ll \nu$ and is $+\infty$ otherwise. The relative Fisher information is $\fisher{\rho}{\pi} := \int \normofLigne{\nabla \log (\der \rho / \der \pi)}^2 \der \rho$.
We note $\tv(\mu, \nu) := \sup_{A \in \borel} |\mu(A) - \nu(A)|$ for the total variation distance.
The set of couplings between $\mu$ and $\nu$ is denoted by $\Gamma(\mu,\nu)$ and $\Wrm_p(\mu,\nu)$ is the $p$-Wasserstein distance.
The Lebesgue measure on $\Rd$ is denoted by $\mathrm{Leb}(\Rd)$. For $\mu \in \probameasures$ and $f \in \Lrm^1(\mu)$, we write $\mu(f) := \int f \der \mu = \Eof[\mu]{f}$.
We denote by $\mu \otimes \nu$ the product of $\mu, \nu \in \probameasures$.
The law of a random variable $X$ is denoted $\law(X)$. Given $\nu \in \probameasures$ and a Markov kernel $K$, we note $\nu \otimes K := \law(X,Y)$, where $X \sim \pi$ and $Y | X = x \sim K(x, \cdot)$.

\section{Technical Background}
\label{sec:technical-background}

In this section, we recall some notions related to Markov kernels and define the \emph{Poissonization} operation. Finally, we give some technical background on PAC-Bayesian bounds.

\subsection{Markov kernels}
\label{sec:markov-kernels-background}

Given a time-homogeneous Markov process $(X_k)_{k\in\N}$ in $\Rd$, the \emph{Markov kernel} $P(x, A)$ describes the probability of observing $X_{k+1}$ in $A$, given that $X_k = x$. 
More precisely, it is a map $P: \Rd \times \bcal(\Rd) \to \R_+$ such that $P(x,\cdot) \in \probameasures$ for all $x \in \Rd$ and the map $x \mapsto P(x,A)$ is measurable for all $A \in \borel$.
Classically $P$ induces maps $P : \probameasures \to \probameasures$ and $P : \Lrm^\infty (\Rd) \to \Lrm^\infty (\Rd)$ defined for $\mu \in \probameasures, ~ A \in \borel$ and $f \in \Lrm^\infty (\Rd)$ by
\begin{align*}
    \mu P (A) := \Eof[X\sim\mu]{P(X,A)}, \quad Pf (x) := \Eof[Y\sim P(x,\cdot)]{f(Y)}.
\end{align*}
These operations are dual from each other, in the sense that $\int f \der (\mu P) = \int Pf \der \mu$,
whenever one of the integrals exists \citep{rudolf_explicit_2012}.
Note that the operator $P : \Lrm^\infty (\Rd) \to \Lrm^\infty (\Rd)$ may be extended outside of $\Lrm^\infty (\Rd)$, when it is well-defined.

A probability measure $\pi$ is \emph{invariant} under $P$ if $\pi P = \pi$, and \emph{reversible} if
\begin{align*}
    \int f P g \der \pi = \int g P f \der \pi, \quad f,g \in \Lrm^\infty (\Rd) .
\end{align*}

\textbf{Dual operator.} Let $\nu \in \probameasures$ and assume that $P$ satisfies the following \emph{continuity preservation} condition: for all $\mu \in \probameasures$, if $\mu \ll \nu$, then $\mu P \ll \nu$. 
Let $f \in \Lrm^1(\nu)$ with $f\geq 0$ and $\int f \der \nu = 1$. This allows us to define
\begin{align}
    \label{eq:l1-adjoint-construction}
    P^\star f := \frac{\der \mu P}{\der \nu}, \quad \der\mu := f \der \nu.
\end{align}
We can easily extend this formula for all $f\in \Lrm^1(\nu)$ by linearity and by decomposing $f$ into positive and negative parts. This gives an operator $P^\star : \Lrm^1(\nu) \longrightarrow \Lrm^1(\nu)$.


\subsection{Poissonization}
\label{sec:poissonization-background}

 Let us first recall the definition of a Poisson process.
\begin{definition}[Poisson process]
    \label{def:poisson-process}
    A Poisson process $(N_t)_{t\geq 0}$ with intensity $\lambda > 0$ is a $\N$-valued Lévy process, almost-surely increasing, with $N_0=0$ and $\Pof{N_t=k} = e^{-\lambda t} {(\lambda t)^k} / {k!}$.
\end{definition}
In all the following, we fix a Poisson process $(N_t)_{t\geq 0}$ with intensity $1$, independent of all other random variables.
We refer to \citep{schilling_introduction_2016} for the definition of a Lévy process.
Given a Markov process $(X_k)_{k \in \mathds{N}}$, the \emph{Poissonized process} \citep{lasota_chaos_1994} is
\begin{align}
    \label{eq:poissonized-process}
    \boxed{
     Y_t := X_{N_t}.
     }
\end{align}
Equivalent definitions of Poissonization exist. For example, \citet{even_continuized_2021} used a stochastic integral formulation in their study of Nesterov acceleration. Poissonization can also be seen as a \emph{continuized semigroup}, \ie, if $P$ is the Markov kernel of $(X_k)_{k\in\N}$, then $(Y_t)_{t\geq 0}$ is a continuous-time Markov process with infinitesimal generator $L := P - I$. This point of view has been adopted by many authors in the Markov chains literature \citep{diaconis_logarithmic_1996,chen_logarithmic_2008,del_moral_contraction_2003,wang_transport-information_2020}.

\subsection{PAC-Bayesian bounds}
\label{sec:PAC-Bayesian-bounds-section}

Analyzing the generalization error of stochastic optimization algorithms leads to the consideration of randomized predictors, which have been classically studied by PAC-Bayesian theory (see \citep{alquier2024user} for an introduction). More precisely, for $S\in\zcal^n$, we define the posterior $\rho_S \in \probameasures$ to be the distribution\footnote{More precisely, $(\rho_S)_{S\in\zcal^n}$ is a Markov kernel on $\zcal^n \times \Rd$.} of the random output of the algorithm given $S$. Given a data-independent prior distribution $\pi \in \probameasures$, a wide variety of PAC-Bayesian bounds have related the generalization error to the KL divergence $\klb{\rho_S}{\pi}$ \citep{mcallester_pac-bayesian_1999,mcallester_pac-bayesian_2003,maurer_note_2004,catoni_pac-bayesian_2007,germain_pac-bayesian_2009,seeger_pac-bayesian_2002} (to name a few). 

In our study, we say that $\ell$ is $\Sigma^2$-subgaussian when for all $\lambda > 0$ and all $w \in \Rd$, we have
\begin{align}
    \label{eq:subgaussian-definition}
    \Eof[Z \sim \mu_z]{e^{\lambda (\ell(w,Z) - \Eof[Z'\sim \mu_z]{\ell(w,Z')})}} \leq e^{\frac{\lambda^2 \Sigma^2}{2}}
\end{align}
To utilize this assumption, we use the PAC-Bayesian bound proposed by \citet{dupuis_generalization_2024}, which is similar to that of \citet{mcallester_pac-bayesian_2003,germain_pac-bayesian_2009} for bounded losses, yielding high-probability bounds that are uniform over the posterior distribution.

\begin{theorem}
\label{thm:subgaussian-pac-bayes}
    We consider a mapping $ S \mapsto \mathscr{P}_S \subseteq \setof{\rho\in\pcal(\Rd),~\klb{\rho}{\pi} < +\infty}$, with $S\in\zcal^n$.
    Assume that $\ell$ is $\Sigma^2$-subgaussian.
    Then, we have
    \begin{align*}
        \Pof[S\sim\datadist]{\forall \rho \in \mathscr{P}_S,~\Eof[W \sim \rho]{G_S(W)} \leq 2\Sigma \sqrt{\frac{\klb{\rho}{\pi} + \log(3/\zeta)}{n}}} \geq 1 - \zeta,
    \end{align*}
    as soon as the set inside the probability $\prob_{S \sim \datadist}$ is $\fcal^{\otimes n}$-measurable.
\end{theorem}

\section{Entropy Flow and Generalization of Poissonized Markov Algorithms}
\label{sec:poissonization-entropy-flow}

In this section, we present our entropy flow computation for Poissonized Markov algorithms.
This finding, presented in \Cref{thm:entropy-flow-weak-regularity} underlies all our main results.
We first discuss the Poissonization of Markov algorithms in \Cref{sec:poissonization-notations-and-convergence} and our assumptions in \Cref{sec:assumptions-entropy-flow}.
In \Cref{sec:entropy-flow-subsection}, we present our general entropy flow derivations.
Finally, \Cref{sec:poisson-sobolev} makes the link with a class of modified logarithmic Sobolev inequalities.

\subsection{Poissonized Markov algorithms}
\label{sec:poissonization-notations-and-convergence}

In what follows, we fix a data-dependent time-homogeneous Markov process $(X_k^S)_{k \in \N}$ in $\Rd$, where $S \in \zcal^n$. 
Several examples (such as SGD and SGLD) will be discussed in \Cref{sec:applications}.
We denote by $P_S$ the Markov kernel of $(X_k^S)_{k \in \N}$, for $S \in \zcal^n$.
We assume that the dynamics is initialized from a Borel probability measure $\mu_0$, \ie, $X_0^S \sim \mu_0$. 

With these notations, our theory aims to obtain generalization bounds for the Poissonized process $(Y_t^S := X_{N_t}^S)_{t\geq 0}$, as defined in \Cref{sec:poissonization-background}. We will denote by $\rho_t^S := \law(Y_t^S)$ the probability distribution of $Y_t^S$, with $t \geq 0$ and $S\in\zcal^n$.
Our goal is to apply the PAC-Bayesian bounds of \Cref{thm:subgaussian-pac-bayes} using $\rho_t^S$ as a posterior distribution. To better exploit the Markovian structure of the problem, we define the prior distribution as follows.
\begin{definition}[Prior distribution]
    \label{def:prior-distribution}
    We call a prior distribution the invariant measure $\pi \in \pcal(\Rd)$ of a data-independent Markov kernel $P$, called the prior Markov kernel.
\end{definition}
In the following, we systematically assume that the initial distribution $\mu_0$ is absolutely continuous with respect to the chosen prior $\pi$, \ie, $\mu_0 \ll \pi$.

\paragraph{Depoissonization.}
The generalization error of the Poissonized process is expressed as
\begin{align}
    \label{eq:poissonized-generalization-formula}
    \Eof{G_S(Y_t^S) | S} = e^{-t} \sum_{k\in\N} \frac{t^k}{k!} \Eof{G_S(X_k^S) | S} \by{$\datadist$-almost surely} .
\end{align}
We show in \Cref{lemma:integrability_of_G_S} that this quantity is well-defined under our assumptions.
Whether this provides pertinent information about the non-Poissonized iterates is a legitimate concern. 
A discrete-time Markov chain and its Poissonized version are known to have comparable properties \citep{jacquet_analytical_1998,levin_markov_2017,caputo_entropy_2024}. That being said, reconstructing the depoissonized distribution of $X_k$ from $Y_t$ is a long-standing technical problem \citep{teugels_note_1972,vallee_depoissonisation_2018,jacquet_analytical_1998}. 

Beyond these classical depoissonization results, we show below that Poissonization provides a sound approximation of the generalization error of convergent Markov algorithms.


\begin{restatable}{theorem}{ThmDepoissonizationInvariant}
    \label{thm:poissonization-invariant-measure-convergence}
    Assume that $|\ell| \leq B < \infty$ and that $\tv(\mu_0 P_S^k, \mu^S) \to 0$ for some $\mu^S \in \probameasures$, \as for $S$.
    Then, \as, $\left| \Eof{G_S(X_k^S) - G_S(Y_k^S) | S} \right| \to  0 $.
    If moreover there exists $C_S>0$ and $a_S \in (0,1)$ such that, \as, $\tv (\mu_0 P_S^k,\mu^S) \leq C_S a_S^k$, then, \as, $$\left|\Eof{G_S(X_k^S) - G_S(Y_k^S) | S} \right| \leq 2 B C_S e^{-(1 - a_S)k}.$$ 
    If $\ell$ is $L$-Lipschitz, then we can replace $\tv$ by the Wasserstein distance $\Wrm_1$ (and $B$ by $L$).
\end{restatable}

The conditions $\tv (\mu_k^S,\mu^S) \leq C_S a_S^k$ (resp. $\Wrm_1 (\mu_k^S,\mu^S) \leq C_S a_S^k$) used above are related to \emph{geometric} (resp. \emph{Wasserstein}) ergodicity \citep{meyn_markov_1993,gallegos_herrada_equivalences_2023} that has been widely studied in the context of convergence of Markov chains \citep{rudolf_perturbation_2017}. 
Note that our condition is weaker than geometric ergodicity: we do not assume the uniqueness of the invariant distribution. 
These concepts have received increasing attention in learning theory due to their connections with SGD \citep{zhu_uniform--time_2023-1} and differential privacy \citep{simsekli_differential_2024}. Although our study is not specific to this exact setting, \Cref{thm:poissonization-invariant-measure-convergence} provides a \emph{sufficient} condition to ensure that Poissonization is a relevant continuous-time approximation of discrete dynamics.

\subsection{Main assumptions and notations}
\label{sec:assumptions-entropy-flow}

Whenever it makes sense, we denote the Radon-Nikodym derivative between the posterior distribution and the prior distribution as
\begin{align}
    \label{eq:vt-definition}
    v_t := \frac{\der \rho_t^S}{\der \pi}, \quad t\geq 0,~ S\in\zcal^n.
\end{align}
Our theory relies on a simple condition, which simply ensures that the posterior and prior Markov kernels ($P_S$ and $P$) are of similar nature, in the sense that $v_t$ cannot explode too fast. This is made precise by the following assumption.

\begin{assumption}[Compatibility condition]
    \label{ass:compatibility-condition}
    Let $\mu_0$ be an initial distribution such that $\klb{\mu_0}{\pi} < +\infty$ and $v_0 := \der \mu_0 / \der \pi$ is bounded away from $0$. 
    We assume that the map $S \mapsto \klb{\rho_t^S}{\pi}$ is $\fcal^{\otimes n}$-measurable.
    Moreover, we assume the following, for all $S \in \zcal^n$
    \begin{enumerate}[label=\textbf{(H.\arabic*)},leftmargin=3\parindent]
        \item \label{ass:continuity-preservation} \textbf{(continuity preservation)} For all $\mu \in \pcal(\Rd)$, if $\mu \ll \pi$, then $\mu P_S \ll \pi$.
        \item \label{ass:bounded-entropy} For all $t \geq 0$, we have $\Phi (P_S^\star v_t) \in \Lrm^1(\pi)$, where $\Phi(x) := x\log(x)$.
    \end{enumerate}
\end{assumption}
The assumption that $\mu_0 \ll \pi$ and the condition \ref{ass:continuity-preservation} are natural in our context, since it ensures that $\rho_t^S \ll \pi$.
We expect it to be mild in practice and, in particular, it holds in the following examples: \textrm{(i)} if $\pi$ is fully supported on a discrete space, \textrm{(ii)} if $P_S$ corresponds to SGD with Gaussian noise addition (\ie, SGLD), \textrm{(iii)} if $\pi \ll \mathrm{Leb}(\Rd)$ and $P_S$ corresponds to SGD on a gradient-Lipschitz loss with small enough learning rate (see \citep[Theorem 3]{clerico_generalisation_2023}). 
Moreover, by \Cref{ass:continuity-preservation}, we can define the adjoint $P_S^\star: \Lrm^1(\pi) \to \Lrm^1(\pi)$ of $P_S$ with respect to $\pi$, in the sense of \Cref{sec:markov-kernels-background}. 

It can be seen from the proof of \Cref{thm:entropy-flow-weak-regularity} that $\klb{\rho_t^S}{\pi}$ is finite under the condition \Cref{ass:bounded-entropy}
(note that this does not imply that $\klb{\rho_t^S}{\pi}$ is small, which is the purpose of the presented theory).
The measurability assumption on $S \mapsto \klb{\rho_t^S}{\pi}$ is mild, as it holds as soon as $\zcal$ is countable, for instance.
We also expect \ref{ass:bounded-entropy} to be mild. 
In particular, we prove that it holds for all the examples discussed in \Cref{sec:sgld,sec:sgd-perturbed-last-iterate,sec:sgd-noise-injection}.
To further illustrate this, we provide a functional characterization of this assumption in \Cref{lemma:continuous-operator-case}, showing that this condition is satisfied as soon as $P_S$ defines a continuous operator $\Lrm^2(\pi) \to \Lrm^2(\pi)$.





\subsection{The entropy flow for Markov algorithms}
\label{sec:entropy-flow-subsection}

In this section, we present our entropy flow formula for Poissonized Markov algorithms, which is the backbone of our contributions.
We denote $\Phi(x) := x \log (x)$ (with $\Phi(0) := 0$).

To compute the derivative of the map $t \mapsto \klb{\rho_t^S}{\pi}$, we first need to understand the time-derivative of $v_t$, defined in \Cref{eq:vt-definition}. The evolution of the probability density of a Poissonized process is well-understood and satisfies a so-called Boltzmann equation \citep{lasota_chaos_1994}. The following lemma is an adaptation of \citep[Equation 8.3.7]{lasota_chaos_1994} to our setup and notation (the main difference being the use of $\pi$ as a reference measure). We provide a rigorous proof of this lemma in \Cref{sec:boltzmann-and-regularity-proof}.

\begin{restatable}[Boltzmann equation]{lemma}{lemmaBoltzmannEquation}
    \label{lemma:boltzmann-with-regularity}
    Assume that the continuity preservation condition  (\Cref{ass:continuity-preservation}) holds, then $v_t$ is solution of the following Boltzmann equation, for all $t>0$,
    \begin{align*}
        \frac{\partial v_t}{\partial t} = \left( P_S^\star - I \right) v_t.
    \end{align*}
\end{restatable}
In our paper, the Boltzmann equation plays an analogous role to the Fokker-Planck equations in the analysis of the Langevin algorithm \citep{mou_generalization_2017}.
Equipped with this lemma, we compute in the next theorem the entropy flow for Poissonized algorithms. 

\begin{restatable}[Entropy flow]{theorem}{thmEntropyFlow}
    \label{thm:entropy-flow-weak-regularity}
    Assume that \Cref{ass:compatibility-condition} holds. We have that
    \begin{align*}
        \timeder \klb{\rho_t^S}{\pi} = \int (P_S - P) (\Phi' \circ v_t) v_t \der \pi - \ecal_{\pi,P} (\Phi'\circ v_t, v_t),\quad t > 0.
    \end{align*}
    where $\ecal_{\pi,P} (f,g) := - \int g (P - I) f \der \pi$ is the \emph{Dirichlet form} associated with the prior process.
\end{restatable}

In the sequel, we use the following notation for the first term of the entropy flow.
\begin{definition}[Expansion term]
    \label{def:expansion-term}
    We call \emph{expansion term} the first term on the right-hand side of \Cref{thm:entropy-flow-weak-regularity}. Whenever it is defined, we will denote it as (note that $\Phi'= 1 + \log$),
        $$
        \Delta_{P,P_S} (f) := \int (P_S - P) (\log f) f \der \pi.
        $$
\end{definition}

 \Cref{thm:entropy-flow-weak-regularity} expresses the entropy flow as the difference between the \emph{expansion} term and the \emph{Dirichlet form}.
 The expansion term represents the discrepancy between the posterior dynamics $P_S$ and the prior one $P$, and \Cref{sec:controlling-discrepancy} is dedicated to its analysis.
 The Dirichlet form is a classical object in the study of Markov processes \citep{bakry_analysis_2014,bottcher_levy_2013}, where it plays a crucial role in the construction of continuous-time semigroups and in the formulation of functional inequalities satisfied by their invariant distribution. 

 By the invariance $\pi$ under $P$, the Dirichlet term $\ecal(\log v_t, v_t)$ can be represented as
     \begin{align}
        \label{eq:dirichlet-bregman-representation}
         \ecal_{\pi, P} (\log v_t, v_t)
         &= \iint \mathrm{D}_{\Phi} (v_t(x), v_t(y)) \der (\pi \otimes P) (x,y) \geq 0,
     \end{align}
 where $\mathrm{D}_\Phi(a,b) := \Phi(a) - \Phi(b) - \Phi'(b) (a - b)$ denotes the Bregman divergence \citep{amari2016information}.
 By the convexity of $\Phi$, we see that $\ecal_{\pi, P} (\log v_t, v_t) \geq 0$. This representation as an integrated Bregman divergence is analogous to the ``Bregman integral'' appearing in \citep{dupuis_generalization_2024} in their study of heavy-tailed SDEs. 

\begin{remark}
    In \Cref{thm:entropy-flow-weak-regularity}, the prior $\pi$ is the invariant measure of a Markov kernel $P$. Alternatively, one can use time-dependent priors $\pi_t$ by defining $\pi_t$ as the law of a Poissonized data-independent Markov process $(X_k)_{k\in\N}$. Then \Cref{eq:dirichlet-bregman-representation} still holds, \ie,
    \begin{align*}
        \timeder \klb{\rho_t^S}{\pi_t} = \Delta_{P,P_S}(v_t) - \iint \mathrm{D}_{\Phi} (v_t(x), v_t(y)) \der (\pi_t \otimes P) (x,y), \quad v_t := \frac{\der \rho_t^S}{\der \pi_t}.
    \end{align*}
\end{remark}

\begin{remark}
    By adapting the regularity assumptions, our theory holds formally for more general convex functions $\Phi$, hence, yielding $\Phi$-entropy flows, which we leave for future work.
\end{remark}

We conclude this subsection with the following theorem, which is the first generic generalization bound that we can deduce from our theory.

\begin{theorem}
    \label{thm:generalization-without-lsi}
    Assume that \Cref{ass:compatibility-condition} holds and that the loss is $\Sigma^2$-subgaussian. Then we have, with probability at least $1 - \zeta$ over $S\sim\datadist$ that for all $T>0$,
    $$
         \Eof{G_S(Y_T^S) | S } \leq \frac{2\Sigma}{\sqrt{n}} \left\{ \int_{0}^T \Delta_{P, P_S}(v_t) \der t +  \klb{\mu_0}{\pi} + \log(3 / \zeta) \right\}^{1 /2}.
    $$
\end{theorem}

\begin{proof}
    By \Cref{thm:subgaussian-pac-bayes}, with probability at least $1-\zeta$ over $S \sim \datadist$ for all $T>0$,
    \begin{align}
        \label{eq:pac-bayes-application-in-proof}
        \Eof{G_S(Y_T^S) | S } \leq 2\Sigma \sqrt{\frac{\klb{\rho_T^S}{\pi} + \log(3 / \zeta)}{n}}.
    \end{align}
    We report in \Cref{sec:rigorous-pac-bayes -application} a justification of \Cref{eq:pac-bayes-application-in-proof}, taking care of measure-theoretic considerations, which we omit here for simplicity. In particular, under \Cref{ass:compatibility-condition}, $\klb{\rho_T^S}{\pi}<+\infty$.
    Fix $S \in \zcal^n$. By \Cref{thm:entropy-flow-weak-regularity}, we have that
    \begin{align}
        \label{eq:simple-proof-step-1}
        \timeder \klb{\rho_t^S}{\pi} = \Delta_{P,P_S}(v_t) - \ecal_{\pi, P}(\log v_t, v_t), \quad t > 0.
    \end{align}
    By \Cref{eq:dirichlet-bregman-representation} and \Cref{lemma:continuity-of-entropy} (proven in the appendix), we deduce that
    \begin{align*}
        \klb{\rho_T^S}{\pi} \leq \klb{\mu_0}{\pi} + \int_0^T  \Delta_{P,P_S}(v_t) \der t.
    \end{align*}
    This concludes the proof.
\end{proof}
\Cref{thm:generalization-without-lsi} shows that obtaining generalization bounds for the Poissonized process reduces to an upper bound on the expansion term $\Delta_{P,P_S}(v_t)$. In \Cref{sec:controlling-discrepancy}, we present several methods to control this term. Before presenting these aspects, we observe that the generalization bound above grows linearly with time even if $\Delta_{P,P_S}(v_t)$ is uniformly bounded.
This suggests that the inequality $\ecal_{\pi,P}(\log v_t, v_t) \geq 0$ can be improved, as explained in the next subsection.

\subsection{Generalization bounds under modified log-Sobolev inequalities}
\label{sec:poisson-sobolev}

We first recall the notion of (classical) LSI (see \citealp{bakry_analysis_2014,chafai_logarithmic_2017} for modern introductions). Let $\nu \in \probameasures$, we associate to $\nu$ the \emph{entropy} functional,
\begin{align}
    \label{eq:entropy-functional-definition}
    \ent[\nu]{f} = \entphi{\nu}{f} := \Eof[\nu]{\Phi(f)} - \Phi(\Eof[\nu]{f}), \text{ with }\Phi(x) = x\log(x).
\end{align}
The entropy functional generalizes the KL divergence, in the sense that $\ent[\nu]{\der \mu / \der \nu} = \klb{\mu}{\nu}$ as soon as $\mu \ll \nu$. 
A probability measure $\nu$ is said to satisfy the $\beta$-LSI if for all positive $f \in \Lrm^1(\nu)\cap \mathcal{C}^1(\Rd)$ we have the inequality,
\begin{align*}
    \ent[\nu]{f} \leq \frac{1}{2\beta} \int \frac{\normof{\nabla f}^2}{f} \der \nu = \frac{1}{2\beta}\ecal^{\mathrm{diff}} (\log f, f),\quad \ecal^{\mathrm{diff}} := \int \langle \nabla f, \nabla g \rangle \der \nu.
\end{align*}

For example, $\mathcal{N}(0,\sigma^2 I_d)$ satisfies a $1/\sigma^2$-LSI \citep{gross_logarithmic_1975-1}.
Such inequalities have been extensively studied for their links with the convergence of Markov processes \citep{bakry_analysis_2014}. In learning theory, they have been used for the generalization analysis of noisy algorithms \citep{mou_generalization_2017} and differential privacy \citep{chourasia_differential_2022,dupuis2025renyidifferentialprivacyheavytailed}.
In the above inequality, $\ecal^{\mathrm{diff}}$ is called the Dirichlet form associated with the Langevin process whose invariant distribution is $\nu$.
This reveals a more general structure and, in fact, variants of LSIs are naturally associated with more general Dirichlet forms \citep{bakry_analysis_2014}.
In our study, we use the following notion of \emph{modified} LSI.

\begin{definition}[Modified LSI]
     \label{def:dirichlet-modified-lsi}
    We say that $(\pi, P)$ satisfies a modified $\gamma$-LSI if for any positive $f$ such that $f \log f \in \Lrm^1(\pi)$, we have
    \begin{align}
        \label{eq:dirichlet-modified-lsi}
        \ecal_{\pi,P}(\log(f),f) \geq \gamma \ent{f}.
    \end{align}
     To make this definition perfectly rigorous, we show in \Cref{lemma:dirichlet-makes-sense} that $\ecal_{\pi,P}(\log(f),f)$ always has meaning in $\R \cup \setof{+\infty}$ under our assumption and that $\ecal_{\pi,P}(\log(f),f) \geq 0$.
\end{definition}

Such inequalities were introduced by \citet{diaconis_logarithmic_1996} and extensively studied  to analyze the convergence rate of Markov chains \citep{bobkov_modified_2006,bobkov_modified_1998,goel_modified_2004,wu_new_2000,ane_logarithmic_2000}.
The term "modified" LSI refers to the use of the term $\ecal(\log (f), f)$ instead of the more classically used $\ecal(\sqrt{f}, \sqrt{f})$ and this terminology has been adopted by many of the aforementioned authors. Both definitions are equivalent in the Gaussian case, and in general, modified LSI is weaker than LSI.

In the next theorem, we exploit modified LSIs to reach a novel generalization bound.

\begin{restatable}[Generalization bound under LSI]{theorem}{thmGenericGeneralization}
\label{thm:generalization-under-modified-lsi}
    In the setting of \Cref{thm:generalization-without-lsi}, assume that $(\pi, P)$ satisfies a modified $\gamma$-LSI. Then, with probability at least $1 - \zeta$, for all $T>0$,
    \begin{align*}
        \Eof{G_S(Y_t^S) | S} \leq \frac{2\Sigma}{\sqrt{n}} \left\{  \int_0^T e^{-\gamma(T - t)}  \Delta_{P,P_S}(v_t)  \der t + e^{-\gamma T} \klb{\mu_0}{\pi} + \log \left( \frac{3}{\zeta} \right) \right\}^{1 / 2} .
    \end{align*}
\end{restatable}

\begin{proof}
    The proof follows the same lines as the proof of \Cref{thm:generalization-without-lsi} until \Cref{eq:simple-proof-step-1} (included). Then, by the modified logarithmic Sobolev inequality, we have that for all $t > 0$,
    \begin{align*}
        \timeder \klb{\rho_t^S}{\pi} \leq \Delta_{P,P_S}(v_t) - \gamma \klb{\rho_t^S}{\pi}.
    \end{align*}
    By \Cref{lemma:continuity-of-entropy} in the appendix and Grönwall's lemma, we conclude that
    \begin{align*}
        \klb{\rho_T^S}{\pi} \leq  e^{-\gamma T} \klb{\mu_0}{\pi} + \int_{0}^T e^{-\gamma (T - t) } \Delta_{P,P_S}(v_t) \der t.
    \end{align*}
    The result follows.
\end{proof}
This shows that priors satisfying modified LSIs induce an exponential decay $e^{-\gamma(T - t)}$ in the bound, improving over \Cref{thm:generalization-without-lsi}. This is analogous to the role of LSIs for Gaussian \citep{mou_generalization_2017} or heavy-tailed \citep{dupuis_generalization_2024} SDEs. 
\Cref{thm:generalization-under-modified-lsi} reduces the problem of controlling time-uniformly $\Eof {G_S(Y_t^S) | S}$ to upper-bounding $\Delta_{P,P_S}(v_t)$. 

To further characterize the modified LSI constant, we define below the entropy contraction coefficient, as suggested by \citet{Raginsky2014StrongDP} (and references therein).

\begin{definition}[Contraction coefficient]
    \label{def:contraction-coeffcient-KL}
    Let $(\pi, P)$ be as in \Cref{def:prior-distribution}. 
    We define $\mathscr{F} := \setof{\nu \in \probameasures,~ \klb{\nu}{\pi} < +\infty}$. We define the entropy contraction coefficient as
    \begin{align*}
        \delta(\pi, P) := \inf_{\nu \in \mathscr{F}} \frac{\klb{\nu}{\pi} - \klb{\nu P}{\pi}}{\klb{\nu}{\pi}}.
    \end{align*}
    By invariance of $\pi$ under $P$ and the data processing inequality, we have $ \delta(\pi, P) \geq 0$.
\end{definition}
There exist many coefficients to measure the contraction properties of Markov chains, such as Dobrushin's coefficient \citep{dobrushin_central_1956}. We invite the reader to consult \citep{caputo_entropy_2024} for a recent exposition of these notions.
Many authors noted that, under structural assumptions on $P$, $\delta(\pi, P)$ is related to the optimal modified LSI constant appearing in \Cref{def:dirichlet-modified-lsi} \citep{Raginsky2014StrongDP,blanca_entropy_2022,del_moral_contraction_2003}. 
In our case, we will make use of the following lemma, for which we provide a rigorous proof in \Cref{sec:proofs-poisson-sobolev-section}.
\begin{restatable}{lemma}{lemmaContractionCoeffLSI}
    \label{lemma:lsi-transfer-general-lemma}
    Let $(\pi, P)$ be as in \Cref{def:prior-distribution} and assume that $\pi$ is reversible for $P$.
    Then $\ecal_{\pi,P}$ satisfies the following modified LSI: for all non-negative $f$ such that $f\log (f) \in \Lrm^1(\pi)$,
    \begin{align*}
       \ecal_{\pi,P} (\Phi'(f), f) \geq \ent{f} - \ent{P f} \geq \delta(\pi, P)\ent{f} .
    \end{align*}
\end{restatable}
The following example is an application of \Cref{lemma:lsi-transfer-general-lemma}, showing that discrete Ornstein Uhlenbeck processes satisfy LSIs. See \Cref{sec:proofs-poisson-sobolev-section} for a proof.

\begin{example}
    \label{ex:ornstein-uhlenbeck-example}
    Consider the discrete Ornstein-Uhlenbeck process defined by $X_{k+1} = (1 - \gamma) X_k + \sigma \xi_k$, with $\sigma > 0$, $\gamma \in (0,1)$, and $(\xi_k)_{k \geq 0} \sim \mathcal{N}(0, I_d)^{\otimes \infty}$. Then, the invariant distribution of this process is $\pi := \mathcal{N}(0,\sigma_\pi^2)$ with $\sigma_\pi := \sigma / \sqrt{1 - (1 - \gamma)^2}$. Then $\pi$ is reversible, and we have $\delta(\pi, P) \geq 2 \gamma - \gamma^2 \geq \gamma$. Thus, by \Cref{lemma:lsi-transfer-general-lemma}, $(\pi, P)$ satisfies a modified LSI with constant $\gamma$.
\end{example}
It is clear from the proof in \Cref{sec:proofs-poisson-sobolev-section} that this lemma can be generalized to obtain modified LSIs for Markov kernels associated with any diffusion process whose reversible measure satisfies a (classical) LSI. Hence, this \Cref{lemma:lsi-transfer-general-lemma} may be seen as a transfer result between classical and modified LSIs.



\section{Controlling the Discrepancy Between Markov Kernels}
\label{sec:controlling-discrepancy}

In \Cref{sec:poisson-sobolev}, we have controlled the Dirichlet form that appears in \Cref{thm:entropy-flow-weak-regularity} through modified LSIs and proved that such inequalities were satisfied by diffusive priors. To apply our theory to practical algorithms, it remains to analyze the expansion term $\Delta_{P,P_S}(v_t)$.
We present two methods to achieve this, depending on the structure of the algorithm. First, we consider noisy algorithms in \Cref{sec:noisy-algorithms} and then turn to non-noisy algorithms in \Cref{sec:singular-algorithms}.

\subsection{Generalization bounds for noisy algorithms}
\label{sec:noisy-algorithms}

The generalization error of noisy iterative algorithms has been extensively studied \citep{haghifam_sharpened_2020,xu_information-theoretic_2017,negrea_information-theoretic_2020,bu_tightening_2020-1}.
A fundamental example is noisy SGD, for which we introduce some notation in the following.

\begin{example}[Noisy SGD]
    \label{ex:noisy-sgd-noisy-algorithm}
    We define noisy SGD by the recursion
    $$ X_{k+1}^S = (1 - \lambda \eta) X_k^S - \eta \widehat{g}_S(X_k^S, U_k) + \xi_k,$$
    with learning rate $\eta > 0$, regularization coefficient $\lambda \geq 0$ (potentially $0$), stochastic gradient $\widehat{g}_S(x, U_k)$, and added noise $\xi_k$, where $(\xi_k)_{k\geq 0}$ are independently and identically distributed (i.i.d.).
    $(U_k)_{k\in \N}$ represent the randomness of the batch indices and are i.i.d. and independent of $(\xi_k)_{k \geq 0}$. We denote $\mu_U$ the law of $U_k$. SGLD corresponds to $\xi_k \sim \mathcal{N}(0, \sigma^2 I_d)$.
\end{example}
With our notation, we characterize a noisy algorithm by the following assumption.
\begin{assumption}[Noisy algorithms]
    \label{ass:noisy-conditions}
    For all $x \in \Rd$ and $S \in \zcal^n$, we have $\delta_x P_S \ll \delta_x P$.
\end{assumption}
Many of our applications use prior kernels such that $\delta_x P \ll \mathrm{Leb}(\Rd)$. Therefore, \Cref{ass:noisy-conditions} describes algorithms that have densities with respect to the Lebesgue measure.
Thus, a typical example of noisy algorithm is \Cref{ex:noisy-sgd-noisy-algorithm} with additional noise $(\xi_k)$ that is absolutely continuous with respect to the Lebesgue measure.


The next proposition presents a generic upper bound on the expansion term for noisy algorithms. The proof can be found in \Cref{sec:proofs-noisy-algorithms}.

\begin{restatable}{proposition}{propExpansionNoisyAlgo}
    \label{prop:expansion-term-general-bound-noisy}
    Assume that \Cref{ass:compatibility-condition,ass:noisy-conditions} hold. Then, we have
    \begin{align}
        \label{eq:expansion-term-noisy-bound}
        \Delta_{P,P_S}(v_t) \leq  \int {\inf_{q > 0} \left(\frac1{q} \klb{\delta_x P_S}{\delta_x P} + q \Irm(q, x, t) \right)} \der \rho_t^S(x),
    \end{align}
    where, using the normalized entropy notation $\overline{\entrm}_\nu (f) := \nu(f)^{-1} \entrm_\nu(f)$, we define
    \begin{align}
        \label{eq:expansion-term-convergence-to-variance}
        \Irm(q, x, t) := \frac1{q} \int_0^q \overline{\entrm}_{\delta_x P}(v_t^\epsilon) \frac{\der \epsilon}{\epsilon^2} \underset{q \to 0}{\longrightarrow}  \frac{1}{2} \var_{\delta_x P} (\log v_t).
    \end{align}
\end{restatable}

\Cref{prop:expansion-term-general-bound-noisy} can be seen as a global-to-local bound, as it transfers the estimation of the ``global'' divergence $\klb{\rho_t^S}{\pi}$ to the ``local'' divergence $\klb{\delta_x P_S}{\delta_x P}$, which in general has a simpler structure. This term quantifies the impact of an individual iteration on the generalization error.
\Cref{eq:expansion-term-noisy-bound} contains a free parameter $q > 0$. This is analogous to the use of Young's inequality in entropy flow derivations along Langevin dynamics \citep[eq. (23)]{mou_generalization_2017} or Lévy-driven SDEs \citep[Cor. 4.2]{dupuis_generalization_2024}. \Cref{prop:expansion-term-general-bound-noisy} is a formal generalization of this proof technique for general Markov processes.

We see in \Cref{eq:expansion-term-convergence-to-variance} that the term $\Irm(q,x,t)$ is of order $1$ in $q$ as $q \to 0^+$ (if $\log(v_t) \in \Lrm^2(\delta_x P)$). Therefore, the local KL divergence dominates the bound as $q \to 0^+$. In \Cref{sec:applications} we show in specific examples how the term $\Irm(q,x,t)$ can be explicitly bounded or even avoided. 
In the next example, we express the local KL divergence for SGLD.

\begin{example}[SGLD]
    \label{ex:sgld-expansion-term-example}
    Consider SGLD as in \Cref{ex:noisy-sgd-noisy-algorithm}, with $\lambda > 0$. Let the prior Markov kernel be $P(x, \cdot) = \mathcal{N}((1 - \lambda \eta)x, \sigma^2 I_d)$. Then, by the joint convexity of the KL divergence \citep{van_erven_renyi_2014}, we have, for $x \in \Rd$, $\klb{\delta_x P_S}{\delta_x P} \leq \frac{\eta^2}{2\sigma^2} \EofLigne{\normofLigne{\widehat{g}_S(x, U)}^2 | S}$.
\end{example}
This example will be used in \Cref{sec:applications} to obtain Poissonized generalization bounds for SGLD.
The following example gives an intuitive interpretation of the variance term appearing in \Cref{eq:expansion-term-convergence-to-variance}, in the case of Gibbs distributions.
\begin{example}[Gibbs distributions]
    \label{ex:gibs-distribution-example}
    Assume that $\der \rho_t^S \propto e^{-V^S_t (x)} \der x$ and $\der \pi \propto e^{-V} \der x$, where $V^S_t, V : \Rd \to \R$ are Gibbs potentials. Then we have
    \begin{align*}
        \var_{\delta_x P} (\log v_t) =  \var_{\delta_x P} \left( V^S_t - V  \right).
    \end{align*}
    Consider the SGLD example of \Cref{ex:noisy-sgd-noisy-algorithm} with $\lambda > 0$ and choose $P(x, \cdot) = \mathcal{N}(\gamma x, \sigma^2 I_d)$ with $\gamma = 1 - \lambda \eta$, then we expect the above quantity to be close to $\var_{\delta_x P} (\er)$ as $t \to \infty$ \citep{raginsky_non-convex_2017}. In this case, we have $\var_{\delta_x P} (\er) = \EofLigne{ ( \er (\gamma x + \xi) - \EofLigne{\er(\gamma x + \xi)} )^2}$ with $ \xi \sim \mathcal{N}(0,\sigma^2 I_d)$.
    This term is similar in nature to the ``expected sharpness'' appearing in \citep[Eq. (8)]{neyshabur_exploring_2017} and the ``local value sensitivity'' used by \citet{neu_information-theoretic_2021}. Therefore, we can interpret this variance term as a measure of the sharpness of the local minimum to which the algorithm has converged.
\end{example}

\subsection{Analysis of non-noisy algorithms}
\label{sec:singular-algorithms}

Unfortunately, various popular procedures, starting with SGD, cannot satisfy \Cref{ass:noisy-conditions}.
In this subsection, we extend our framework beyond noisy algorithms. 

To avoid the condition that $\delta_x P_S \ll \delta_x P$, we perform an expansion of $P_S$ ``around $P$''. This is made precise by the following proposition, which is inspired by \citep[Proposition $1$]{polyanskiy_wasserstein_2016} and \citep[Lemma 3.5]{raginsky_non-convex_2017}. 

\begin{restatable}{proposition}{propScoreWasserstein}
    \label{prop:first-term-bound-f-regular}
    Assume that, for all $t\in [0,T]$, there exist constants $c_1,c_2>0$ such that $v_t$ satisfies the linear growth condition: $\forall x \in \Rd,~\normof{\nabla \log v_t(x)} \leq c_1 \normof{x} + c_2$, then we have.
    \begin{align*}
        \Delta_{P,P_S}(v_t) \leq \Eof[Y \sim \rho_t^S]{\Wrm_2(\delta_Y P, \delta_Y P_S)^2}^{\frac1{2}}\left(\frac{c_1}{2} \normof{P}_t + \frac{c_1}{2}\normof{P_S}_t + c_2 \right),
    \end{align*}
    with $\normof{P}_t^2 := \Eof[X \sim \rho_t^S P]{\normof{X}^2}$ (resp. $P_S$) and $\Wrm_2$ the Wasserstein distance (\Cref{sec:proofs-singular-algorithms}).
\end{restatable}

The main feature of \Cref{prop:first-term-bound-f-regular}, proven in \Cref{sec:proofs-singular-algorithms}, is to relate the expansion term to the Wasserstein distance $\Wrm_2(\delta_Y P_S, \delta_Y P)$. Note that this term is local, \ie, is computed for every point $Y\in\Rd$ and is expected over the posterior distribution. In the case of SGD, with a relevant choice of $P$, we usually have $\Wrm_2(\delta_Y P, \delta_Y P_S)^2 \leq \eta^2 \Eof[U]{\normof{\widehat{g}_S(Y,U)}^2}$, with the notation of \Cref{ex:noisy-sgd-noisy-algorithm}. Similar Wasserstein distances between Markov kernels have been extensively studied in the context of convergence and geometry of Markov chains \citep{rudolf_perturbation_2017,ollivier_ricci_2007} and have been used by \citep{zhu_uniform--time_2023-1} to obtain stability-based bounds for SGD.
Therefore, \Cref{prop:first-term-bound-f-regular} connects our framework and the prior art through these Wasserstein terms.

\Cref{prop:first-term-bound-f-regular} is based on a condition of linear growth of $\nabla \log v_t = \nabla \log u_t^S - \nabla \log \pi$.
This assumption can be seen as an assumption on the tail of the posterior density $u_t^S$ and can hint towards good choices of prior, \ie, with tails similar to the posterior density. 
Finally, note that a related condition was used by \citet[Section A.4]{li_generalization_2020}.



\section{Applications}
\label{sec:applications}

In this section, we apply the results of \Cref{sec:poissonization-entropy-flow,sec:controlling-discrepancy} to obtain generalization bounds in concrete examples. We start by showing in \Cref{sec:sgld} that our techniques recover known rates in the case of SGLD, and then show their ability to reach novel generalization bounds.



\subsection{Sanity check: Stochastic gradient Langevin dynamics}
\label{sec:sgld}

In this subsection, we apply our framework to the SGLD algorithm, using the notations of \Cref{ex:noisy-sgd-noisy-algorithm}. According to \Cref{sec:poissonization-notations-and-convergence}, we denote the Poissonized SGLD process as $Y_t^S := X_{N_t}^S$. To apply our theory to SGLD, we rely on \Cref{thm:entropy-flow-weak-regularity} and bound the expansion term using the same proof technique as in the proof of \Cref{prop:expansion-term-general-bound-noisy}. In the particular case of SGLD, we can exploit the specific structure of diffusion processes to simplify \Cref{eq:expansion-term-noisy-bound} even further. This leads to the following bound, whose proof is postponed to \Cref{sec:proofs-sgld-bounds}.
 
\begin{restatable}[Poissonized SGLD]{theorem}{thmSGLDKLBound}
    \label{thm:sgld-kl-bound}
    Suppose that $\eta\lambda < 1$ and take $(P,\pi)$ as in \Cref{ex:ornstein-uhlenbeck-example}.
    Let $\pi$ be the invariant distribution of this process.
    Assume that \Cref{ass:compatibility-condition} hold and the loss is $\Sigma^2$-subgaussian.
    Then, with probability at least $1 - \zeta$ over $S \sim \datadist$, for all $T>0$,
    \begin{align*}
        \Eof{G_S(Y_t^S) | S} \leq \frac{2 \Sigma}{\sqrt{n}} \left( \frac{\eta^2 (2 - \lambda \eta)}{2  \sigma^2} \int_0^T e^{-\lambda \eta (T - t)}\Eof{\normof{\widehat{g}_S(Y_t^S,U)}^2 | S}  \der t + K_T + \log \frac{3}{\zeta} \right)^{1 / 2},
    \end{align*}
    with $K_T := e^{-\lambda\eta T} \klb{\mu_0}{\pi}$, for some $U\sim\mu_U$ independent of $Y_t^S$.
\end{restatable}
\Cref{thm:sgld-kl-bound} is proven in \Cref{sec:proofs-sgld-bounds}. We also prove in \Cref{lemma:compatibility-for-sgld} that \Cref{ass:compatibility-condition} is satisfied in the particular case where $\nabla \ell(\cdot, z)$ is Lipschitz and that $\der \mu_0 / \der \pi$ is bounded away from zero.
Although Poissonization might not be the preferred framework to prove generalization bounds for SGLD, we observe that the orders of magnitude of the different terms in \Cref{thm:sgld-kl-bound} are the same as the results of \citet{mou_generalization_2017} obtained under similar assumptions (it should be noted that our result does not require a Lipschitz loss).
Indeed, since the Poisson process $(N_t)_{t\geq 0}$ has intensity $1$, the integral $\int_0^K$ for $K\in\N$ should be intuitively thought of as the Poissonized counterpart of a sum $\sum_{k=1}^K$, where $K$ is the number of iterations.
Therefore, our framework is general enough to recover Poissonized counterparts of classical results, hence unifying the proof techniques of several generalization bounds.

\begin{remark}[Dissipativity assumption]
    \label{rq:dissopativity-assumptions}
    Our proof technique also applies to SGLD without regularization, at the cost of assuming the dissipativity of the loss, \ie, $\langle w, \nabla \ell (w,z) \rangle \geq m \normof{w}^2 - b$, for some $m>0$ and $b \geq 0$. Such dissipativity assumptions are common in this literature \citep{futami_time-independent_2023,raginsky_non-convex_2017,farghly_time-independent_2021}. 
\end{remark}

\subsection{SGD with perturbed final iterate}
\label{sec:sgd-perturbed-last-iterate}

We now apply our framework to obtain novel generalization bounds for SGD. We consider the regularized SGD recursion with the notation of \Cref{ex:noisy-sgd-noisy-algorithm}, \ie, $X_{k+1}^S = (1 - \lambda \eta) X_k^S - \eta \widehat{g}_S(X_k^S,U_k)$. We adopt a technique similar to \citep{neu_information-theoretic_2021, neyshabur_exploring_2017} to obtain information-theoretic bounds for SGD, consisting in perturbing the \emph{last} SGD iterate with Gaussian noise. Hence, instead of observing $X_k^S$ directly, we consider the random variables, $\overline{X}_k^S := X_k^S + \sigma \xi_k$, with $(\xi_k)_{k\in\N} \sim \mathcal{N}(0,I_d)^{\otimes \infty}$ independent of $(X_k^S)_{k\in\N}$ and $\sigma>0$.
According to \Cref{sec:poissonization-notations-and-convergence}, we actually consider a Poissonized version of this stochastic process, 
\begin{align*}
 \Ybar_t^S := \Xbar_{N_t}^S = Y_t^S + \sigma \xi_{N_t},
\end{align*}
Our goal is to obtain a bound on $\mathds{E} \big[{G_S(\Ybar_T^S) | S} \big]$ for all time horizon $T>0$, with high probability over $S \sim \datadist$.
We show in \Cref{lemma:perturbed-poissonized-process-structure} that the operations of Poissonization and addition of Gaussian noise $(\xi_k)_{k\in\N}$ commute, \ie, $\Ybar_t^S = Y_t^S + \sigma \Xi$, with $\Xi \sim \mathcal{N}(0, I_d)$ independent of $Y_t^S$.
Thus, the (Poissonized) generalization error is written as
\begin{align*}
    G_S(Y_t^S) = G_S(\Ybar_t^S) + \delta_t^S, \quad \delta_t^S := G_S(\Ybar_t^S) -  G_S(Y_t^S),
\end{align*}
where $\delta_t^S$ is formally similar to the difference of \emph{local value sensitivities} appearing in \citep[Theorem 1]{neu_information-theoretic_2021}. It is also related to the \emph{expected sharpness} of \citet{neyshabur_exploring_2017} (to be precise, this term is linked to $\er(\Ybar_t^S) -  \er(Y_t^S)$).
We refer to these works for further analysis of these terms and focus our analysis on the perturbed Poissonized process.

The main difficulty arising is that the process $(\Xbar_k^S)_{k\in\N}$ defined above is not Markovian in general. This means that the posterior density cannot satisfy the same Boltzmann equation as \Cref{lemma:boltzmann-with-regularity}. Nevertheless, we adopt a proof technique that relies on noting that the posterior density actually satisfies a \emph{perturbed} Boltzmann equation. This observation, which is explained in \Cref{sec:proofs-sgd-perturbed-last-iterate}, allows us to exploit our framework to obtain a new generalization bound for SGD.
We make the following assumptions on the loss function.

\begin{assumption}
    \label{ass:perturbed-sgd-assumptions}
    The stochastic gradient norms are bounded by a fixed constant $L>0$, \ie, $\normof{\widehat{g}_k} < L$.
     The initialization $v_0 := \der \mu_0 / \der \pi$ has compact support included in $B(0,L/\lambda)$ and there exists $r > 0$ such that $v_0$ is bounded away from $0$ on $B(0,r)$.
\end{assumption}
This assumption justifies the regularity conditions necessary for our entropy flow derivations.
The following theorem is a generalization bound for perturbed SGD.

\begin{restatable}[Generalization bound for perturbed SGD]{theorem}{thmGeneralizationPerturbedSGD}
    \label{thm:sgd-perturbed-last-iterate}
    Take $\pi := \mathcal{N}(0, \sigma^2 I_d)$.
    Assume that \Cref{ass:perturbed-sgd-assumptions} holds, that the loss $\ell(w,z)$ is $\Sigma^2$-subgaussian with respect to $z \sim \mu_z$, and that the initial distribution $\mu_0$ has its support included in the ball $B(0, L / \lambda)$. Then, we have with probability at least $1 - \zeta$ over $S\sim\datadist$ that, for all $T>0$,
    \begin{align*}
        \Eof{G_S(\bar{Y}_T^S) | S} \leq \frac{2\Sigma}{\sqrt{n } } \left\{ \int_0^T \frac{\eta L e^{-\lambda \eta(T -t)} }{\lambda \sigma^2}\sqrt{\Eof{\normof{\widehat{g}_S (Y_t^S, U)}^2 \big| S} } \der t + K_T + \log \frac{3}{\zeta} \right\}^{1 / 2},
    \end{align*}
    with $K_T := e^{-\lambda \eta T} \klb{\mu_0}{\pi}$ and $U\sim\mu_U$ is independent of $Y_t^S$.
\end{restatable}
\Cref{thm:sgd-perturbed-last-iterate} shows that the generalization of the perturbed last iterate of SGD is upper bounded by a weighted integral of the gradient norms encountered during training. Due to the presence of exponential decay, $e^{-\lambda \eta(T -t)} $, in the integral, this bound puts more weight on the stochastic gradient norms toward the end of training. This observation is consistent with the popular idea that generalization is improved when the algorithm converges to flat minima \citep{jastrzebski_three_2018-1,keskar_large-batch_2017}.
\Cref{thm:sgd-perturbed-last-iterate} should be compared to \citep[Theorem 1]{neu_information-theoretic_2021}, which can be informally summarized with our notation as
\begin{align*}
    \Eof[S,U]{G_S(\Xbar_K^S)} \lesssim \frac{\Sigma}{\sqrt{n}} \left\{ \frac1{\sigma^2}  \sum_{k=1}^K \eta^2 \Eof{\normof{\widehat{g}_S(X_k^S) - \nabla \risk (X_k^S)}^2} \right\}^{1 / 2} ,
\end{align*}
where we omitted the sensitivity terms.
Our bound directly involves the gradient norm $\normof{\widehat{g}_S}^2$ instead of its distance from the population gradient. Although the latter might be smaller in some cases, it cannot be estimated from the training dataset.
Moreover, our framework naturally provides high probability bounds, while \citet{neu_information-theoretic_2021} only prove expected bounds.
Most importantly, \Cref{thm:sgd-perturbed-last-iterate} shows that, in the presence of regularization, the sum over all iterations can be replaced by an integral with an exponential decay term, which significantly improves the time dependence of the bound.
Note, however, that our theory only provides a generalization bound for the Poissonized iterates $\Ybar^S_t$, see \Cref{sec:poissonization-notations-and-convergence}.

\subsection{SGD with noise injection for strongly convex losses}
\label{sec:sgd-noise-injection}

It was suggested by several authors that noise injection inside the evaluation of the stochastic gradient can lead to improved generalization \citep{orvieto_anticorrelated_2023,liu_noisy_2021,nesterov_smooth_2005}. To the best of our knowledge, no explicit generalization bound has been derived for this algorithm. In this subsection, we show that our framework naturally leads to a generalization bound in this setting.
More precisely, we study the following regularized gradient descent algorithm with noise injection,
\begin{align*}
    X_{k+1}^S = X_k^S - \eta \nabla \er (X_k^S + \sigma \xi_k), \quad (\xi_k)_{k\in \N} \sim \mathcal{N}(0, I_d)^{\otimes n},
\end{align*}
where $\xi_k$ is independent of $X_k^S$. As before, we denote the Poissonized process by $Y_t^S := X_{N_t}^S$. Its probability distribution is denoted $\rho_t^S$ and we write $\rhobar_t^S := \rho_t^S \ast \mathcal{N}(0,\sigma^2 I_d)$.

According to \citet{orvieto_explicit_2023}, this recursion is implicitly regularized toward flat minima. Indeed, by Itô's lemma, we see that if $\er$ is twice differentiable, we have
\begin{align*}
    \Eof{\er(x + \sigma \xi_k)|S} = \er(x) + \frac1{2}\int_0^{\sigma^2} \Eof{\Delta \er (x + \sqrt{u} \xi_k)| S} \der u \approx \er(x) + \frac{\sigma^2}{2} \Delta \er(x).
\end{align*}
Therefore, the above recursion is analogous to using the trace of the Hessian matrix of $\er$ as a regularizer, hence pushing the dynamics toward flatter regions, which are commonly associated with better generalization error \citep{jastrzebski_three_2018-1,keskar_large-batch_2017}.

Thanks to our framework, we easily obtain a generalization bound for this algorithm in the case of a strongly convex loss function.

\begin{restatable}{theorem}{thmSGDNoiseInjection}
    \label{thm:noise-injection-strongly-convex}
    Assume that $\ell(\cdot, z) \in \mathcal{C}^2(\Rd)$ and is $\gamma$-strongly convex (with $\gamma \eta < 1$) and gradient-Lipschitz for all $z \in \zcal$ and that $\ell(w, z)$ is $\Sigma^2$-subgaussian with respect to $z \sim \mu_z$. Then, we have with probability at least $1 - \zeta$ over $S\sim\datadist$ that
    \begin{align*}
        \Eof{G_S(Y_t^S) | S} \leq \frac{2 \Sigma}{\sqrt{n}} \left\{ \int_0^T e^{-\eta \gamma (T - t)} \Eof[\rhobar_t^S]{2\Delta \el + \frac1{\sigma^2} \normof{\nabla \el}^2} \der t +  K_T + \log \frac{3}{\zeta} \right\}^{1 / 2},
    \end{align*}
    where $\el(y) = \gamma^{-1}\er(y) -  \normof{y}^2 / 2$, $K_T = e^{-\gamma \eta T} \klb{\mu_0}{\pi}$, and $\pi = \mathcal{N}(0, \sigma^2 \gamma \eta / (2 - \gamma \eta) I_d)$.
\end{restatable}
To our knowledge, \Cref{thm:noise-injection-strongly-convex} is the first generalization bound for this algorithm, which confirms the regularization effect of noise injection.
The assumption that $\nabla \ell (\cdot,z)$ is Lipschitz-continuous is only meant to ensure that \Cref{ass:compatibility-condition} holds in this case. In particular, the value of this smoothness constant does not appear in the final statement.
We note that our generalization bound is expressed in terms of two quantities, the gradient norms and the Laplacian of the empirical loss, expected over the perturbed Poissonized process $Y_t^S + \sigma \Xi$, with $\Xi \sim \mathcal{N}(0,I_d)$. Compared to the results of \Cref{sec:sgd-perturbed-last-iterate}, this shows that noise injection connects the generalization error to the curvature of the loss landscape.


\subsection{Toward generalization bounds for SGD under linear growth assumption}
\label{sec:sgd-linear-growth}

We consider the regularized SGD recursion as defined in \Cref{sec:sgd-perturbed-last-iterate} (with the notation of \Cref{ex:noisy-sgd-noisy-algorithm}), $X_{k+1}^S := (1 - \lambda \eta) X_k^S - \eta \widehat{g}_S(X_k^S, U_k)$.
We denote $\rho_t^S$ the distribution of the Poissonized process $Y_t^S$ and, when defined, its density with respect to the Lebesgue measure is denoted $u_t^S$.
The following result is a consequence of \Cref{prop:first-term-bound-f-regular} and \Cref{thm:entropy-flow-weak-regularity}.

\begin{restatable}{theorem}{thmSGDWithWasserstein}
    \label{thm:sgd-application-wasserstein}
    Assume that $\ell(w,z)$ is $\Sigma^2$-subgaussian and that \Cref{ass:compatibility-condition} is satisfied with $\pi := \mathcal{N}(0, \sigma^2 I_d)$, for some $\sigma>0$. For simplicity, set $\mu_0 = \pi$ and assume that there exist constants $a, b > 0$ such that for all $t \in [0, T]$, we have $\normof{\nabla \log u_t^S (x) + x / \sigma^2} \leq a \normof{x} + b$.
    Then, we have, with probability at least $1 - \zeta$ over $S \sim \datadist$, that for all $T>0$,
    \begin{align*}
        \Eof{G_S(Y_T^S) | S} \leq \frac{2\Sigma}{\sqrt{n}} \left\{\int_{0}^T e^{-\lambda \eta (T - t) } \mathrm{N}_S(t)  \left( a \mathrm{N}_S(t) + a \mathrm{M}_S(t)  + b \right) \der t + \log \frac{3}{\zeta} \right\}^{1 / 2},
    \end{align*}
    with $\mathrm{N}_S(t)^2 := \eta^2{\EofLigne{\normofLigne{\widehat{g}_S(Y_t^S, U)}^2 | S}}  + \sigma^2 {\lambda\eta d}$ and $\mathrm{M}_S(t)^2 := 2{\EofLigne{\normofLigne{Y_t^S}^2 | S}}$.
\end{restatable}
This theorem illustrates the flexibility of our framework, which allows to obtain many generalization bounds for both noisy algorithms (\Cref{sec:sgld,sec:sgd-perturbed-last-iterate}) and non-noisy algorithms such as SGD under specific assumptions, while relying on the same unified proof technique.
 We see that \Cref{thm:sgd-application-wasserstein} relates the generalization error of SGD to the stochastic gradient norms, averaged over the posterior distribution. Similar quantities appear classically in the study of noisy SGD \citep{mou_generalization_2017,negrea_information-theoretic_2020,haghifam_sharpened_2020,dupuis_generalization_2024} and were already involved for non-noisy SGD \citep{neu_information-theoretic_2021}. Compared to \citep{neu_information-theoretic_2021}, the main advantage of \Cref{thm:sgd-application-wasserstein} is the presence of the exponential decay term, $e^{-\lambda \eta (T - t)}$.
The assumption that $\normof{\nabla \log u_t^S (x) + x / \sigma^2} \leq a \normof{x} + b$ is meant to ensure the linear growth condition of \Cref{prop:first-term-bound-f-regular}, which is the main ingredient in the proof of \Cref{thm:sgd-application-wasserstein}. 
Intuitively, this assumption implies that the SGD iterations do not converge too fast to singular distributions, which we believe might be related to the edge of stability phenomenon
\citep{cohen2021gradient}. 
This shows that the regularity of the distributions that are induced by the algorithm impact its generalization properties.

\section{Conclusion and Future Work}

In this paper, we introduced a new framework to extend the entropy flow method to general time-homogeneous Markov algorithms. 
Drawing inspiration from the Poissonization operation, we derive an exact entropy flow formula for Poissonized Markov algorithms which we connected to modified logarithmic Sobolev inequalities, hence, relating the generalization error of general Markov algorithms to ergodic theory.
We provided generic tools to upper bound the expansion term that appears in our theory, both for noisy and non-noisy algorithms.
Thanks to this new framework, we easily recovered (Poissonized versions of) classical generalization bounds for SGLD and obtained new generalization bounds for several algorithms, such as SGD and variants of SGD with noise injection.

\paragraph{Future work.} 
Several directions remain to be explored with our new entropy flow method.
As mentioned in the introduction, the entropy flow is a popular method to analyze the differential privacy of Langevin Monte Carlo algorithms or differentially-private SGD with Gaussian noise. 
Therefore, it would be an interesting direction for future work to apply our framework to obtain new differential privacy guarantees.
Moreover, modified LSIs are a popular tool to analyze the convergence properties of discrete Markov chains.
Thus, one potential direction for future research is to apply our work to understand the generalization error of Markov algorithms in discrete parameter spaces.

\acks{U.\c{S}. is partially supported by the French government under the management of
Agence Nationale de la Recherche as part of the ``Investissements d'avenir'' program, reference
ANR-19-P3IA-0001 (PRAIRIE 3IA Institute). B.D., M.H., and U.\c{S}. are partially supported by the European Research Council Starting Grant
DYNASTY – 101039676.}


\appendix

\section{Omitted proofs of \Cref{sec:poissonization-entropy-flow}}
\label{sec:proofs-poissonization-entropy-flow}

Before presenting the proofs of \Cref{sec:poissonization-entropy-flow}, we give more details on the construction of the dual operator in \Cref{eq:l1-adjoint-construction}.
We observe that $(P^\star)^k = (P^k)^\star$, which we will denote $P^{\star k}$, without ampiguity.
Note that the operator $P^\star$, which we will call the dual operator, is neither Markovian nor sub-Markovian in general, and it differs from the dual operators constructed by \citet{del_moral_contraction_2003}. In particular, it does not satisfy Jensen's inequality. Nevertheless, the operator $P^\star$ satisfies the following duality formula,
\begin{align}
    \label{eq:duality-formula}
    \int f P g \der \nu = \int g P^\star f \der \nu,
\end{align}
which is true as soon as $f \in \Lrm^1(\nu)$ and $g P^\star f \in \Lrm^1(\nu)$, by basic properties of Markov kernels.

\subsection{Proof of \Cref{lemma:boltzmann-with-regularity} and additional technical lemmas}
\label{sec:boltzmann-and-regularity-proof}


\begin{proof} \textbf{(of \Cref{lemma:boltzmann-with-regularity})}
    Let $v_0 := \der \mu_0 / \der \pi$.
    For any Borel set $A \in \borel$, we have
    \begin{align*}
        \Pof{X^S_{N_t} \in A}&= \Pof{\bigcup_{k \in \N} \setof{X_k^S \in A, N_t = k}} \\
        &= \sum_{k \in \N} e^{-t} \frac{t^k}{k!} \mu P_S^k (A) \by{disjoint union and independence}.
    \end{align*}
    By definition of the dual operator and Tonelli's theorem, we have
    \begin{align*}
        \Pof{X^S_{N_t} \in A} = e^{-t} \sum_{k\in\N} \frac{t^k}{k!} \int_A P_S^{\star k} v_0 \der \pi = \int_A \bigg( e^{-t} \sum_{k\in\N} \frac{t^k}{k!}  P_S^{\star k} v_0 \bigg)\der \pi.
     \end{align*}
     Therefore, we have, by definition of the Radon-Nikodym derivative,
     \begin{align}
    \label{eq:poissonized-density-expression}
    \boxed{
    v_t(x) := \frac{\der \rho_t^S}{\der \pi} (x) = e^{-t} \sum_{k\in\N} \frac{t^k}{k!} P_S^{\star k} v_0(x) = e^{-t} \sum_{k\in\N} \frac{t^k}{k!} \frac{\der \law (X_k^S)}{\der \pi} (x).
    }
\end{align}
     By the same reasoning and using \Cref{eq:duality-formula}, we show that $P_S^\star v_t(x) = e^{-t} \sum_{k\in\N} \frac{t^k}{k!} P_S^{\star (k+1)} v_0(x).$
     Let $b > a > 0$ and $t\in[a,b]$, for $k \in \N^\star$ we have $|\partial_t (t^k P_S^{\star k} v_0 / k!)| \leq b^{k-1} P_S^{\star k} v_0 / (k-1)!$ . Therefore, as $P_S^\star k v_b \in \Lrm^1(\pi)$, we can differentiate under the sum to obtain
     \begin{align*}
         \frac{\partial v_t}{\partial t} = -v_t + e^{-t} \sum_{k\geq1} \frac{t^{k-1}}{(k-1)!} P_S^{\star k} v_0(x) = -v_t + P_S^\star v_t.
     \end{align*}
     This completes the proof.
\end{proof}
\textbf{Poissonized semigroup}
A formally similar computation provides the expression of the Poissonized Markov semigroup $(P_S)_t$ of $Y_t^S$, for any bounded Borel-measurable function $f$,
\begin{align*}
    (P_S)_t f (x) := \Eof{f(Y_t^S) | Y_0^S = x} = e^{-t} \sum_{k\in\N} \frac{t^k}{k!} P_S^k (x).
\end{align*}
Thus, $((P_S)_t)_{t\geq 0}$ forms a continuous-time semigroup. From this formula and the Boltzmann equation, we easily see that the infinitesimal generator of this semigroup is given by $L_S := P_S - I$, so it can be formally written as $(P_S)_t = e^{t (P_S - I)}.$
See \citep{diaconis_logarithmic_1996} for more details and \citep{bakry_analysis_2014} for an introduction to Markov semigroups.

We conclude this subsection by showing that, under our assumptions, the generalization error is almost-surely integrable with respect to the Poissonized posterior distribution.

\begin{lemma}
    \label{lemma:integrability_of_G_S}
    Assume that \Cref{ass:compatibility-condition} holds and that $\ell(w, z)$ is $\Sigma^2$-subgaussian. Then, for $\datadist$-almost all $S$, we have $G_S \in \Lrm^1(\rho_t^S)$ for all $t\geq 0$. Moreover, \Cref{eq:poissonized-generalization-formula} holds. 
    In particular, the map $t \mapsto \int G_S \der \rho_t^S$ is almost-surely continuous on $\R_+$.
\end{lemma}

\begin{proof}
    Let us fix $t\geq 0$ and $S\in\zcal^n$. By Donsker-Varadhan's formula, we have
    \begin{align*}
       \int {\min(N, |G_S|)} \der \rho_t^S \leq \klb{\rho_t^S}{\pi} + \log \int e^{\min(N, |G_S|)} \der \pi.
    \end{align*}
    By the monotone convergence theorem, we deduce the following inequality in $\R\cup\{+\infty\}$,
    \begin{align*}
         \int {|G_S|} \der \rho_t^S \leq \klb{\rho_t^S}{\pi} + \log \int e^{ |G_S|} \der \pi.
    \end{align*}
    Now, by Tonelli's theorem and the subgaussian condition, we have $\int e^{ |G_S|} \der (\datadist\otimes\pi) < +\infty.$
    Therefore, the term $\int e^{ |G_S|} \der \pi$ is $\datadist$-almost-surely finite. As shown in the proof of \Cref{thm:entropy-flow-weak-regularity}, \Cref{ass:compatibility-condition} ensures that $\klb{\rho_t^S}{\pi}$ is finite for all $S \in \zcal^n$.
    Then, Tonelli's theorem ensures that, for $\datadist$-almost all $S$, we have that $\sum_{k\in\N} {t^k} \Eof{|G_S(X_k^S)| ~|S} / k! < \infty$. By Fubini's theorem, \Cref{eq:poissonized-generalization-formula} therefore holds for $\datadist$-almost all $S$. 
    The desired continuity property then follows from standard power series arguments.
    This concludes the proof.
\end{proof}



\subsection{Proof of \Cref{thm:poissonization-invariant-measure-convergence}}
\label{sec:proof-depoissonization-tv-wasserstein}



\begin{proof} \textbf{(of \Cref{thm:poissonization-invariant-measure-convergence}).}
    \textbf{First part of the statement.} Let us fix $S\in\zcal^n$ such that the assumed convergences almost surely hold ($\tv(\mu_0 P_S^k,\mu^S) \longrightarrow 0$).
    For any $A\in\borel$,
    \begin{align*}
        \rho_t^S(A) = e^{-t} \sum_{k\in\N} \frac{t^k}{k!} (\mu_0 P_S^k)(A).
    \end{align*}
    As $\mu_0 P_S^k$ converges to $\mu^S$ in total variation, for all $\varepsilon> 0$ there exists $K \in \mathds{N}$, such that for all $k\geq K$ and all $A \in \borel$,  $|\mu_k^S(A) - \mu^S(A)| \leq \varepsilon$. Therefore, we have
    \begin{align*}
    |\rho_t^S(A) - \mu^S(A)| \leq \varepsilon + 2 e^{-t} \sum_{0 \leq k \leq K - 1} \frac{t^k}{k!}, \quad A \in \borel,
    \end{align*}
    where the last term is smaller than $\varepsilon$ for all $t$ greater than some $t_0(K)$, depending only on $K$. This shows that $\tv(\rho_t^S, \mu_S) \to 0$, hence, by the triangle inequality, we get
    \begin{align*}
        \left|\Eof{G_S(X_k^S) - G_S(Y_k^S) | S} \right|
        \leq \normof{\ell}_\infty (\tv(\mu_0 P_S^k,\mu^S) + \tv(\mu^S, \rho_k^S)) \underset{k\to\infty}{\longrightarrow}  0 .
    \end{align*}
    Now we assume that there exists $C_S>0$ and $a_S \in (0,1)$ such that $\forall k \in \N,~ \tv (\mu_0 P_S^k,\mu^S) \leq C_S a_S^k$, then we have, for any $A\in\borel$:
    \begin{align*}
        |\rho_t^S(A) - \mu^S(A)| \leq e^{-t} \sum_{k\in\N} \frac{t^k}{k!}  |\mu_k^S(A) - \mu^S(A)| \leq C_S e^{-t} \sum_{k\in\N} \frac{(a_S t)^k}{k!} = C_S e^{-(1 - a_S) t}.
    \end{align*}
    Thus, by the triangle inequality for total variation, we get, for all $k \in \N$,
    \begin{align*}
        \tv(\mu_0 P_S^k, \rho_k^S) \leq \tv(\mu_0 P_S^k, \mu^S) + \tv(\rho_k^S, \mu^S) \leq C_Sa_S^k + C_Se^{-(1 - a_S) k} \leq 2C_Se^{-(1 - a_S) k} ,
    \end{align*}
    where we used that $a \leq e^{-(1 - a)}$, hence, $\Eof{|G_S(X_k) - G_S(Y_k)| \big\vert S} \leq 2 \normof{\ell}_\infty C_Se^{-(1 - a_S) k}$.

    \textbf{Second part of the statement.} We now assume that $\ell$ is $L$-Lipschitz continuous and use Wasserstein distance instead of $\tv$. We sketch the proof as it is similar to the previous case. The main argument, is that by convexity of the Wasserstein distance $\Wrm_1$ (see \citep[Lemma 2.3]{farghly_time-independent_2021} and \citep[Theorem 4.8]{villani_optimal_2009}), we have
    \begin{align}
    \label{eq:wasserstein-convexity-poissonization}
    \Wrm_1(\rho_t^S, \mu_S) =  \Wrm_1 \big(e^{-t} \sum_{k\in\N} \frac{t^k}{k!} \mu_0 P_S^k, e^{-t} \sum_{k\in\N} \frac{t^k}{k!} \mu^S \big)
    \leq e^{-t} \sum_{k\in\N} \frac{t^k}{k!} \Wrm_1(\mu_k^S, \mu_S).
    \end{align}
Then, if we assume that $\Wrm_1(\mu_0 P_S^k, \mu^S) \longrightarrow 0$ and fix $\varepsilon > 0$, we know that there exists $K\in\N$ such that $\forall k \geq K,~ \Wrm_1(\mu_0 P_S^k, \mu^S) \leq \varepsilon$ and we obtain that $\Wrm_1(\rho_t^S, \mu_S) \longrightarrow 0$ by noting that
\begin{align*}
    \Wrm_1(\rho_t^S, \mu_S) \leq \varepsilon + e^{-t} \big(\max_{0\leq k \leq K}  \Wrm_1(\mu_0 P_S^k, \mu^S) \big) \sum_{0 \leq k \leq K - 1} \frac{t^k}{k!} \longrightarrow 0.
\end{align*}
We conclude by Kantorovich duality \citep[Theorem 5.10]{villani_optimal_2009} and the triangle inequality.

Finally, if $\Wrm_1(\rho_t^S, \mu_S) \leq C_S a_S^k$ with $C_S>0$ and $a_S \in (0,1)$, \Cref{eq:wasserstein-convexity-poissonization} gives $\Wrm_1(\rho_t^S, \mu_S) \leq C_Se^{-(1 - a_S)t}$. We conclude again by duality and the triangle inequality.
\end{proof}

\subsection{The case of continuous Markov operators}
\label{sec:proof-continuous-operator-case}


\begin{restatable}{lemma}{lemmaPSContinuousOperator}
    \label{lemma:continuous-operator-case}
    Assume that condition \ref{ass:continuity-preservation} in \Cref{ass:compatibility-condition} is satisfied. Then condition
    \Cref{ass:bounded-entropy} is satisfied if $P_S$ defines a continuous operator $\Lrm^2(\pi) \to \Lrm^2(\pi)$.
\end{restatable}

\begin{proof}
    As $\Lrm^2(\pi)$ is a Hilbert space, we can apply the Riesz theorem to define an adjoint operator $P_S^\star : \Lrm^2(\pi) \to \Lrm^2 (\pi)$, which is also continuous. Let $\normof{P_S^\star}$ denote its operator norm. We can verify that, in this case, $P_S^\star|_{\Lrm^1(\pi)}$ corresponds to the dual operator defined above. 

    By continuity of $P_S^\star$ and the triangle inequality, we have that $v_t \in \Lrm^2(\pi)$, indeed,
    \begin{align*}
    \normof{v_t}_{\Lrm^2(\pi)} \leq  e^{-t} \sum_{k\in\N} \frac{t^k}{k!} \normof{P_S^{\star k} v_0}_{\Lrm^2(\pi)} \leq e^{-t} \sum_{k\in\N} \frac{t^k}{k!} \normof{P_S^\star}^k \normof{ v_0}_{\Lrm^2(\pi)} \leq e^{\normof{P_S^\star}} \normof{ v_0}_{\Lrm^2(\pi)} < \infty.
    \end{align*}
    By continuity, we also have $P_S^\star v_t \in \Lrm^2(\pi)$.
    Moreover, the map $x\mapsto \Phi(x) = x \log(x)$ is lower bounded, and we have $\Phi(P_S^\star v_t) \leq (P_S^\star v_t)^2 \in \Lrm^2(\pi) \subset \Lrm^1(\pi)$. This completes the proof. 
\end{proof}

\subsection{Proof of the entropy flow formula}
\label{sec:entropy-flow-derivation}



Recall that $\Phi(x) := x \log(x)$ ($\Phi(0) = 0$). Throughout the proofs we use without explicit mention the elementary fact that if $0\leq a \leq b$, then $|\Phi(a)| \leq (1/e) + |\Phi(b)|$. Moreover, we can show, by Jensen's inequality that $\Phi (a + b) \leq \Phi(a) + \Phi(b) + (a+b) \log(2)$, for any $a,b \in \R_+$. This justifies some of the domination arguments below.

\begin{proof} \textbf{(of \Cref{thm:entropy-flow-weak-regularity})}
    Let $0 < a < b$ and $t\in[a,b]$. 
    Note that $\Phi$ is lower 
    bounded and
    \begin{align*}
        v_t = e^{-t} \sum_{k \in \N} \frac{t^k}{k!} P_S^{\star k} v_0 \leq e^{-a} v_0 + e^{b - a} e^{-b} \sum_{k\geq 1} \frac{b^k}{k!} P_S^{\star k} v_0 \leq e^{-a} v_0 + b e^b P_S^\star v_b.
    \end{align*}
    Therefore, we have $\Phi(v_t) \in \Lrm^1(\pi)$, by \Cref{ass:compatibility-condition}. 
    Moreover, we have $\partial_t \Phi (v_t) = \Phi'(v_t) (P_S^\star - I)v_t$. 
    We note that $x\Phi'(x) = x + \Phi(x)$ and $\Phi$ is increasing on $[1, +\infty)$. Thus,
    \begin{align*}
        |v_t \Phi'(v_t)| \leq e^{b - a} v_b + \frac1{e}  + |\Phi(e^{b - a} v_b)| \leq \frac1{e} + e^{b - a} v_b + |\Phi(e^{b - a})| v_b + e^{b-a} |\Phi(v_b)| \in \Lrm^1(\pi).
    \end{align*}
    Thus, $v_t \Phi'(v_t)$ is uniformly $\pi$-integrable for $t \in [a,b]$, \ie, $\sup_{a \leq t \leq b} |v_t \Phi'(v_t)| \in \Lrm^1(\pi).$   
    
    For the other term, we have by convexity that $\Phi' (v_t) P_S^\star v_t \leq \Phi(P_S^\star v_t) - \Phi(v_t) + v_t \Phi'(v_t).$
    We have that $P_S^\star v_t \leq e^{b - a} P_S^\star v_b$. Moreover, as $\Phi'$ is increasing, we have $\Phi'(v_t)P_S^\star v_t \geq P_S^\star v_t \Phi'(e^{-a} v_0)$. We assumed that $v_0$ is bounded away from $0$, therefore we conclude by using the previous reasoning that $\Phi' (v_t) P_S^\star v_t$ is also uniformly $\pi$-integrable for $t \in [a,b]$.

    This ensures that we can differentiate under the integral, which gives, by \Cref{lemma:boltzmann-with-regularity},
    \begin{align*}
        \timeder \entphi{\pi}{v_t} = \timeder \int \Phi(v_t) \der \pi &= \int \Phi'(v_t) \left( P_S^\star - I\right) v_t \der \pi.
    \end{align*}
    The reasoning above also shows that $\Phi'(v_t) P_S^\star v_t \in \Lrm^1(\pi)$. This ensures that the duality formula, \Cref{eq:duality-formula} holds, and we have
    \begin{align*}
         \timeder \entphi{\pi}{v_t} &= \int (P_S - I) \Phi'(v_t)  v_t \der \pi.
    \end{align*}
    Now, let $t \in [a,b]$ as before and let $x,y \in \Rd$. By convexity, we have the elementary inequality $v (\log (v) - \log(u)) - (v - u) \geq 0$, thus,
    \begin{align*}
    v_t(x) \Phi'(e^{-b} v_0(y)) \leq v_t(x) \Phi'(v_t (y)) \leq \Phi(v_t(x)) + v_t(y),
    \end{align*}
    By invariance of $P$ under $\pi$ and the fact that $v_0$ is bounded away from $0$, this ensures that $(x,y) \mapsto v_t(x) \Phi'(v_t (y)) \in \Lrm^1(\pi\otimes P)$. Therefore, $v_t P (\Phi'\circ v_t) \in \Lrm^1(\pi)$ and we have
    \begin{align*}
         \timeder \entphi{\pi}{v_t} &= \int (P_S - P) \Phi'(v_t)  v_t \der \pi - \int v_t (I - P) \Phi'(v_t) \der \pi \\
          &= \int (P_S - P) \Phi'(v_t)  v_t \der \pi - \ecal_{\pi, P} (\Phi' \circ v_t , v_t), 
    \end{align*}
    where the last line follows from the definition of $\ecal_{\pi, P}$. This concludes the proof.
\end{proof}


\begin{lemma}
    \label{lemma:continuity-of-entropy}
    Under \Cref{ass:compatibility-condition}, the map $t \mapsto \int \Phi(v_t) \der \pi$ is continuous on $\R_+$. 
\end{lemma}

\begin{proof}
    For all $x \in \Rd$, the continuity of $t \mapsto v_t(x)$ on $\R_+^\star$. comes from the differentiability of $v_t$, obtained as part of the proof of the Boltzmann equation in \Cref{lemma:boltzmann-with-regularity}.    
    We obtained in the proof of \Cref{lemma:boltzmann-with-regularity} that $\Phi(v_t)$ is uniformly integrable for $t \in [a,b]$ for any $b > a \geq 0$. 
    By the dominated convergence theorem, we deduce the desired continuity property.
\end{proof}

\subsection{Rigorous justification of \Cref{eq:pac-bayes-application-in-proof}}
\label{sec:rigorous-pac-bayes -application}

To justify \Cref{eq:pac-bayes-application-in-proof}, we apply \Cref{thm:subgaussian-pac-bayes} with $\mathscr{P}_S := \setof{\rho_T^S,~T \in \Q_{\geq 0}}$. 
By \Cref{ass:compatibility-condition}, \Cref{lemma:integrability_of_G_S}, and the Markov kernel property of $\rho_t^S$, we can show (up to an almost sure equal modification of $G_S$) that the maps $S \mapsto \int G_S \der \rho_t^S$ and $S\mapsto \klb{\rho_t^S}{\pi}$ are $\fcal^{\otimes n}$-measurable.
As $\mathscr{P}_S$ is parameterized by a countable family, we can apply \Cref{lemma:continuity-of-entropy} to argue that the maps $t \mapsto \int G_S \der \rho_t^S$ and $t\mapsto \klb{\rho_t^S}{\pi}$ are continuous on $\R_+$. This leads to \Cref{eq:pac-bayes-application-in-proof}.

In the rest of the paper, we rely on this justification to directly apply \Cref{thm:subgaussian-pac-bayes} in the form of \Cref{eq:pac-bayes-application-in-proof}.
Note that these measure-theoretic complications can be easily avoided in the case where the data space is assumed to be countable.

\subsection{Omitted proofs of \Cref{sec:poisson-sobolev}}
\label{sec:proofs-poisson-sobolev-section}

We first prove that the Dirichlet form can always be defined under our assumptions.

\begin{lemma}
    \label{lemma:dirichlet-makes-sense}
    Let $\Phi(x) := x \log(x)$ ($\Phi(0) := 0$) and let $f$ be a non-negative function such that $\Phi \circ f \in \Lrm^1(\pi)$. Then $\ecal_{\pi,P}(\log f, f)$ always makes sense in $\R\cup\setof{+\infty}$ and $\ecal_{\pi,P}(\log f, f) \geq 0$.
\end{lemma}

\begin{proof}
    Note that
    \begin{align*}
      \ecal_{\pi,P}(\log f, f) = \iint  f(x) [ \log(f(x))  - \log(f(y)) ] \der (\pi \otimes P).
    \end{align*}
    let $x,y \in \Rd$, by the convexity of $\Phi$, we have, $ f(x) [ \log(f(x))  - \log(f(y)) ]  - (f(x) - f(y)) \geq 0$, where the inequality is in $\R\cup\setof{+\infty}$.
    Let $\psi(x,y) := - (f(x) - f(y)) $ and $\phi := f(x) [ \log(f(x))  - \log(f(y)) ]$. As $\Phi(f) \in \Lrm^1(\pi)$, we have that $f \in \Lrm^1(\pi)$, and also $Pf \in \Lrm^1(\pi)$, by invariance of $\pi$. Thus, $\psi \in \Lrm^1(\pi \otimes P)$. This implies that the negative part $\phi^-$ of $\phi$ satisfies $\phi^- \in \Lrm^1(\pi \otimes P)$. Thus, $\ecal_{\pi,P}(\log f, f)$ always makes sense in $\R\cup\setof{+\infty}$. Moreover, the inequality also implies the non-negativity of $\ecal_{\pi,P}(\log f, f)$, by invariance of $\pi$. This concludes the proof.
\end{proof}

We now present the proof of \Cref{lemma:lsi-transfer-general-lemma}.


\begin{proof} \textbf{(of \Cref{lemma:lsi-transfer-general-lemma}).}
    Let $f$ be such that $\Phi \circ f \in \Lrm^1(\pi)$.   
    By reversibility of $\pi$, we have
    \begin{align*}
        \ecal_{\pi,P} (\Phi'\circ f, f) = \int f (I - P) (\Phi'\circ f) \der \pi 
        = \int \left( f \Phi'(f) - \Phi'(f) Pf  \right) \der  \pi.
    \end{align*}
    By convexity of $\Phi$, we have for all $x \in \Rd$ that
    \begin{align*}
        \Phi (Pf(x)) - \Phi(f(x)) - \Phi'(f(x)) (Pf(x) - f(x)) \geq 0.
    \end{align*}
    Therefore, $f \Phi'(f) - \Phi' (f) Pf \geq \Phi (f) - \Phi (Pf) $. By Jensen's inequality and the invariance of $\pi$, we have $\Phi \circ Pf \in \Lrm^1(\pi)$. As $\Phi \circ f \in \Lrm^1(\pi)$, we also have $f \in \Lrm^1(\pi)$.
    Thus,
    \begin{align*}
        \ecal_{\pi,P} (\Phi'\circ f, f) \geq \entphi{\pi}{f} - \entphi{\pi}{Pf}.
    \end{align*}
    Now we note that the entropy functional is homogeneous, in the sense that for any $\lambda > 0$, we have $\ent{\lambda f} = \lambda \ent{f}$. Let us define a probability measure $\mu$ by $\der \mu := f \der \pi / \pi(f)$, with $\pi(f) := \int f \der \pi$. By reversibility of $\pi$ under $P$, we have that $\der (\mu P) / \der \pi  = P (\der \mu / \der \pi)$, $\pi$-almost surely.
    Therefore, by definition of the contraction coefficient, we have
    \begin{align*}
        \ecal_{\pi,P} (\Phi' \circ f, f) \geq \pi(f) \left( \klb{\mu}{\pi}  - \klb{\mu P}{\pi} \right) &\geq \pi(f) \delta_{\Phi}(\pi,P) \klb{\mu}{\pi}  \\&
        = \delta_{\Phi}(\pi,P) \entphi{\pi}{f}.
    \end{align*}
    This concludes the proof.
\end{proof}

We conclude this section by formally prove the statements of \Cref{ex:ornstein-uhlenbeck-example}.

\begin{proof} (of \Cref{ex:ornstein-uhlenbeck-example})
    Consider the Ornstein-Uhlenbeck process, $\der Z_t = -\frac1{\sigma_\pi^2} Z_t \der t + \sqrt{2} \der B_t$,
    where $B$ is a standard Brownian motion. Let $(P_t)_{t\geq 0}$ be the associated semigroup. Define $c := 1 / \sigma_\pi^2$ and $ t_0 := -\sigma_\pi^2 \log (1 - \gamma)$.
    By the Mehler's formula for Ornstein-Uhlenbeck processes, we have 
    that the Markov operator $P$ of the process $(X_k)_{k \geq 1}$ is equal to $P_{t_0}$. Indeed, we have that
    $Z_{t_0}$ is distributed according to $\mathcal{N}(e^{-ct_0} Z_0, ((1 - e^{-2ct_0}) / c) I_d) = \mathcal{N}((1 - \gamma) Z_0, \sigma^2 I_d)$ (see \citep{chafai_logarithmic_2017}), conditioned on $Z_0$. Therefore, the reversibility of $P$ then follows from the reversibility of the invariant measure for the semigroup $(P_t)_{t\geq 0}$. Moreover, $\pi$ satisfies the (classical) logarithmic Sobolev inequality with constant $1 / \sigma_\pi^2$. As a consequence, we have the exponential convergence of the entropy \citep{bakry_analysis_2014}, $\ent{P_t f} \leq e^{-2 t} \ent{f}$. By \Cref{lemma:lsi-transfer-general-lemma} we have $\ecal_{\pi,P}(\log f, f) \geq \delta(\pi, P)\ent{f} $, and
    $$
        \delta(\pi, P)\ent{f} = \ent{f} - \ent{P_{t_0} f} \geq \left( 1 -  e^{-2 t_0 / \sigma_\pi^2 } \right)\ent{f} = (2\gamma - \gamma^2) \ent{f}.
    $$
    This concludes the proof.
\end{proof}

\section{Omitted proofs of \Cref{sec:controlling-discrepancy}}
\label{sec:proofs-controlling-discrepancy}

\subsection{Omitted Proofs of \Cref{prop:expansion-term-general-bound-noisy}}
\label{sec:proofs-noisy-algorithms}


\begin{proof} \textbf{(of \Cref{prop:expansion-term-general-bound-noisy})}
    By \Cref{ass:compatibility-condition} and \Cref{eq:poissonized-density-expression}, $v_t$ is bounded away from $0$. Moreover, by the proof of \Cref{thm:entropy-flow-weak-regularity}, $v_t P\log v_t \in \Lrm^1(\pi)$. This implies that $\log v_t \in \Lrm^1(\delta_x P)$ for $\rho_t^S$-almost all $x \in \Rd$. Let's fix such an $x$.
    By Donsker-Varadhan's formula,
    \begin{align*}
        P_S (\log v_t) (x) \leq \frac1{q} \klb{\delta_x P_S}{\delta_x P} + \frac1{q} \log P (v_t^q)(x) \by{for all $q>0$}.
    \end{align*}
    Therefore, we introduce the function $ \Lambda (q, \cdot, t) :=  \frac1{q} \log P (v_t^q) - P(\log v_t)$,
    where the variable $x \in \Rd$ has been omitted for clarity.
    Then we define $\Irm(q,x,t) := q^{-1} \Lambda (q,x,t)$.
    
    \textbf{Case 1:} Let $q > 0$ and first assume that $v_t^q \in \Lrm^1(\delta_x P)$. Then, for all $\epsilon \in [0, q)$, we have $\partial_\epsilon (v_t^\epsilon) = \log(v_t) v_t^\epsilon$.
    By \Cref{ass:compatibility-condition} ($v_0$ is bounded away from $0$) and the proof of \Cref{lemma:boltzmann-with-regularity}, we know that $v_t$ is lower bounded by some constant $C_t$. Therefore, we have $\partial_\epsilon (v_t^\epsilon) \geq - |\log C_t| (1 + v_t^q)\in \Lrm^1(\delta_x P)$, where we have noted that $v_t^\epsilon \leq 1 + v_t^q$. We also have
    \begin{align*}
        \partial_\epsilon (v_t^\epsilon) = \frac1{q - \epsilon} \log (v_t^{q - \epsilon}) v_t^\epsilon \leq \frac{v_t^q}{q - \epsilon}  \in \Lrm^1(\delta_x P).
    \end{align*}
    By taking $\epsilon$ in a small open interval included in $[0,q)$, this justifies that we can differentiate $\epsilon \mapsto \log P (v_t^\epsilon)$ under $P$. We easily see that $\partial_\epsilon \log P (v_t^\epsilon)|_{\epsilon=0} = P(\log v_t)$.
    Therefore, for all $x\in\Rd$ and all $t> 0$, $\Lambda (\epsilon,x,t) \to 0$ as $\epsilon \to 0$.
    It also justifies that
    \begin{align*}
        \frac{\partial}{\partial \epsilon} \Lambda (\epsilon,\cdot,t) &= \frac{-1}{\epsilon^2} \log P(v_t^\epsilon) + \frac1{\epsilon} \frac{P(v_t^\epsilon \log v_t)}{P(v_t^\epsilon)} \\
        &= \frac{\entphi{P}{v_t^\epsilon}}{\epsilon^2 P(v_t^\epsilon)} \\
        &= \frac1{\epsilon^2} \entphibar{P}{v_t^\epsilon}.
    \end{align*}
    We obtain similarly $\Lambda (q, \cdot, t) \sim_{q \to 0} \frac{q}{2} \mathrm{Var}_{P}(\log v_t)$, justifying the last part of the statement.
    
    \textbf{Case 2:} Now assume that $P(v_t^{q})(x) = +\infty$. By Fatou's lemma, $\liminf_{\epsilon \to q} \log P (v_t^\epsilon) = +\infty$.
    Therefore, the inside of the infimum in the right-hand side of \Cref{eq:expansion-term-noisy-bound} becomes infinite, so the bound holds trivially.
\end{proof}

\subsection{Omitted proofs of \Cref{prop:first-term-bound-f-regular}}
\label{sec:proofs-singular-algorithms}

Before stating the proof of \Cref{prop:first-term-bound-f-regular}, let's recall the definition of Wasserstein distance.

\begin{definition}[Wasserstein distance]
    \label{def:wasserstein}
    let $\mu, \nu \in \probameasures$ with second-order moments and denote by $\Gamma(\mu,\nu)$ the set of \emph{couplings} between $\mu$ and $\nu$. The Wasserstein's distance $\Wrm_2$ is
    $$
        \Wrm_2(\mu,\nu)^2 := \inf_{\gamma \in \Gamma(\mu,\nu)} \left\{ \iint \normof{x - y}^2 \der \gamma(x,y)  \right\}.
   $$
\end{definition}


The proof of \Cref{prop:first-term-bound-f-regular} is inspired by \citep[Proposition 1]{polyanskiy_wasserstein_2016}. 

\begin{proof} \textbf{(of \Cref{prop:first-term-bound-f-regular})}
Let $\gamma_x \in \Gamma (\delta_x P_S, \delta_x P)$ be an optimal coupling for $\Wrm_2 (\delta_x P, \delta_x P_S)$ and $(U,V) \sim \gamma_x$. We denote by $\rho_t^S \otimes \gamma_x$ the joint distribution of $x \sim \rho_t^S$ and $(U,V) \sim \gamma_x$.
By the linear growth condition, we have almost surely that
\begin{align*}
    \log v_t(U) - \log v_t (V) &= \int_0^1 \langle \nabla \log v_t (s U + (1 - s)V), U - V \rangle \der s \\
    &\leq \normof{U - V} \left( \frac{c_1}{2} \normof{U} + \frac{c_1}{2} \normof{V} + c_2 \right). \by{Cauchy-Schwarz's inequality}
\end{align*}
By integrating over $\rho_t^S \otimes \gamma_x$ and using Cauchy-Schwarz and triangle inequalities we obtain
\begin{align*}
    \Delta_{P,P_S}(v_t) &= \mathds{E}_{Y \sim \rho_t^S}\Eof[(U,V) \sim \gamma_Y]{\log v_t(U) - \log v_t(V)} \\
                    &\leq \Eof[Y \sim \rho_t^S]{\Wrm_2 (\delta_Y P, \delta_Y P_S)^2}^{\frac1{2}}\normof{\frac{c_1}{2} \normof{U} + \frac{c_1}{2} \normof{V} + c_2}_{\Lrm^2(\rho_t^S \otimes \gamma_Y)} \\
                    &\leq \Eof[Y \sim \rho_t^S]{\Wrm_2 (\delta_Y P, \delta_Y P_S)^2}^{\frac1{2}}\left(\frac{c_1}{2} \normof{P}_t + \frac{c_1}{2}\normof{P_S}_t + c_2 \right).
\end{align*}
This concludes the proof.
\end{proof}

\section{Omitted proofs of \Cref{sec:applications}}
\label{sec:proofs-applications}

\subsection{Omitted proofs of \Cref{sec:sgld}}
\label{sec:proofs-sgld-bounds}

Before proving \Cref{thm:sgld-kl-bound}, we establish a more general result in \Cref{prop:sgld-exponential-decay}.
We first define a class of Markov kernel $P$ that generalizes \Cref{ex:ornstein-uhlenbeck-example}.

\begin{definition}
    \label{def:representable-kernels}
    Consider a twice differentiable gradient-Lipschitz potential $V:\Rd \to \R$ such that $e^{-V} \in \Lrm^1(\Rd)$, $e^{-V}$ has finite moments of order $2$, and the SDE $\der Z_t = -\nabla V (Z_t) \der t + \sqrt{2} \der B_t$,
    where $(B_t)_{t\geq 0}$ is a standard Brownian motion. We say that a Markov kernel $P$ is represented by this equation at time $t_0 > 0$ if $P(x,.) = \mathrm{Law}(Z_{t_0} | Z_0 = x)$ for all $x \in \Rd$. 
\end{definition}

\begin{proposition}
    \label{prop:sgld-exponential-decay}
    Assume that $P$ is representable in the sense of \Cref{def:representable-kernels}. We denote $(P_t)_{t\geq 0}$ the semigroup of the SDE appearing in \Cref{def:representable-kernels} and $\pi$ its invariant distribution.
    Assume further that $\nabla^2 V \succeq K I_d$ ($K>0$) and that \Cref{ass:compatibility-condition} holds. 
    \begin{align*}
        \klb{\rho_T^S}{\pi} \leq \frac1{q} \int_0^T e^{-\frac{T - \tau}{\tau_0}}\Eof[x\sim\rho_\tau^S]{\klb{\delta_x P_S}{\delta_x P}}  \der \tau + e^{-T / \tau_0} \klb{\mu_0}{\pi}.
    \end{align*}
    with the constants $q$ and $\tau_0$ given by
    \begin{align*}
        q =  \frac{1 - e^{-Kt_0} }{1 - e^{-2Kt_0}}, \quad   \frac1{\tau_0} = 1 - e^{-Kt_0}.
    \end{align*}
\end{proposition}

\begin{proof}
    We denote $\Phi(x) := x\log (x)$ ($\Phi(0) = 0$).
    The proof technique is a refinement of the proof of \Cref{prop:expansion-term-general-bound-noisy}, exploiting the structure of the underlying diffusion process. Let us denote by $\pi$ the reversible (and invariant) distribution of $(P_t)_{t\geq 0}$ and by $\lcal$ its infinitesimal generator. As before, we denote $L = P - I$ ($L$ is the generator of the Poissonized semigroup while $\lcal$ is the generator of the semigroup $(P_t)_{t\geq 0}$).
    By \Cref{thm:entropy-flow-weak-regularity}, we have that
    \begin{align*}
        \frac{\der}{\der \tau} \klb{\rho_\tau^S}{\pi} = \Eof[\pi]{v_\tau(P_S - P) \log v_{\tau}} - \ecal_\pi (\log v_{\tau}, v_{\tau}),
    \end{align*}
    where we used $\tau$ as the time variable to avoid later confusion with $(P_t)_{t\geq 0}$.

    By \Cref{ass:compatibility-condition}, we know that $v_t \log P v_t$ is bounded away from $0$. The invariance of $\pi$ and the Donsker-Varadhan's formula ensure that $\rho_t^S (\log v_t) \leq \klb{\rho_t^S}{\pi}$, so that $v_t \log P v_t \in \Lrm^1(\pi)$. By Jensen's inequality, we also have $v_t \log P v_t^q \in \Lrm^1(\pi)$.
    
     By the Donsker-Varadhan formula, we have, for all $\tau > 0$, $S \in \zcal^n$, $q \in (0,1]$ and $x \in \Rd$,
    \begin{align}
        \label{eq:sgld-decay-step-1}
        P_S \log v_\tau (x) = \frac1{q} P_S \log \left(  v_\tau^q\right)(x) \leq \frac1{q} \klb{\delta_x P_S}{\delta_x P} + \frac1{q} \log \left( P (v_\tau^q)(x) \right).
    \end{align}   
    By Hölder's inequality, $v_\tau^q \in \Lrm^1(\delta_x P)$ for $\pi$-almost all $x \in \Rd$.
    Consider a differentiable function $q : [t_0/2, t_0] \to (0,1]$, (determined later).
    Let $f : \Rd \to \R$ be a $\mathcal{C}^2$ positive function s.t. $\forall x \in \Rd, ~ \log f \in \Lrm^1(\delta_x P)\cap \Lrm^1(\delta_x P_S)$, $f\log f \in \Lrm^1(\delta_x P)$ and $f$ is bounded away from $0$. We first assume that $f$ has bounded derivatives of order $0$, $1$, and $2$. Let
    \begin{align*}
        \alpha (t) := \frac1{q(t)} \log P_t \left( f^{q(t)} \right).
    \end{align*}
    Let us denote $g_t := P_t \left( f^{q(t)} \right)$, by the chain rule we have:
    \begin{align*}
        \alpha'(t) = - \frac{q'(t)}{q(t)^2} \log g_t + \frac1{q(t)} \frac{\lcal g_t}{g_t} + \frac{q'(t)}{q(t)} \frac{P_t \left( f^{q(t)} \log(f) \right)}{g_t}.
    \end{align*}
    Let us denote $\Gamma$ the carré du champ operator, \ie, in our case $\Gamma(\psi) := \normof{\nabla \psi}^2 $, see \citep{bakry_analysis_2014}. By the diffusion property (see \citep[Section 5.4]{chafai_logarithmic_2017}), we have
    \begin{align*}
        \alpha'(t) &= - \frac{q'(t)}{q(t)^2} \log g_t + \frac1{q(t)} \bigg( \lcal \log g_t + \frac{\Gamma(g_t)}{g_t^2} \bigg) + \frac{q'(t)}{q(t)} \frac{P_t \left( f^{q(t)} \log(f) \right)}{g_t} \\
        &= \frac1{q(t)} \left( \lcal \log g_t + \frac{\Gamma(g_t)}{g_t^2} \right) + \frac{q'(t)}{q(t)^2} \frac{\ent[P_t]{f^{q(t)}}}{g_t}.
    \end{align*}
    Now assume that $\forall t\geq 0, ~ q'(t) \leq 0$. By \citep[Lemma 5.6]{chafai_logarithmic_2017}, $\nabla^2 V \succeq K I_d$ implies that the semigroup $(P_t)_{t\geq 0}$ satisfying the $\cd(K,\infty)$ condition. Thus, by the reverse local logarithmic Sobolev inequality \citep[Theorem 5.5.2 (v)]{bakry_analysis_2014}, we have
    \begin{align*}
        \alpha'(t) \leq \frac1{q(t)} \left( \lcal \log g_t + \frac{\Gamma(g_t)}{g_t^2} \right) + \frac{q'(t)}{q(t)^2} \frac{e^{2Kt} - 1}{2K} \frac{\Gamma(g_t)}{g_t^2}.
    \end{align*}
    Based on this inequality, we choose the function $q$ on $t \in [t_0 / 2, t_0]$ to be (recall that $P = P_{t_0}$),
    \begin{align*}
       q(t) := \exp \left\{ - \int_{\frac{t_0}{2}}^{t} \frac{2K}{e^{2Ku} - 1} \der u \right\}.
    \end{align*}
    This leads to the differential inequality, for all $t_0 / 2 \leq t \leq t_0$, $\alpha'(t) \leq \frac1{q(t)} \lcal \log(g_t) = \lcal \alpha(t)$.
    By the chain rule, we can now write, for $t_0 / 2 \leq s \leq t_0$, that
    \begin{align*}
        \frac{\der}{\der s} \left( P_{t_0 - s} \alpha(s) \right) = - \lcal P_{t_0 - s} \alpha(s) + P_{t_0 - s} \alpha'(s) \leq - \lcal P_{t_0 - s} \alpha(s) + P_{t_0 - s} \lcal \alpha(s)  = 0,
    \end{align*}
    where we used that the semigroup $(P_t)_{t\geq 0}$ commutes with its infinitesimal generator $\lcal$. Therefore, the map $s \to P_{t_0 - s} \alpha(s) $ is decreasing. In particular, by using $q(t_0/2) = 1$ and the interpolation condition $P = P_{t_0}$ we have
    \begin{align}
        \label{eq:sgld-proof-crucial-step}
       \frac1{q(t_0)} \log P \left( f^{q(t_0)} \right) =  \alpha (t_0) \leq P_{\frac{t_0}{2}} \alpha \left( \frac{t_0}{2} \right) = P_{\frac{t_0}{2}} \log P_{\frac{t_0}{2}} f.
    \end{align}
    So far we assumed that $f$ has bounded derivatives of order $0$, $1$, and $2$.
    By \Cref{ass:compatibility-condition} and the construction of $v_\tau$, we know that $v_\tau \geq c_\tau$ for some $c_\tau > 0$.
    As $v_\tau \in \Lrm^1(\delta_x P)$, using classical density arguments (see \citep{bogachev_measure_2007} for instance), we can approximate $v_\tau$ in $\Lrm^1(\delta_x P)$ by a sequence of functions $(c_\tau + \varphi_N)_{N \in \N}$, with $\varphi_N \in \mathcal{C}_c^\infty (\Rd)$ and $\varphi_N \geq 0$. As convergence in $\Lrm^1(\delta_x P)$ implies convergence in probability for $\delta_x P$, up to extracting a subsequence, we can assume that $c_\tau + \varphi_N \to v_\tau$, $\delta_x P$-almost surely. Then, we can apply \Cref{eq:sgld-proof-crucial-step} to $c_\tau + \varphi_N$ and apply Fatou's lemma to get
    \begin{align*}
        \frac1{q(t_0)} \log P \left( v_\tau^{q(t_0)} \right) \leq \liminf_{N\to \infty} P_{\frac{t_0}{2}} \log P_{\frac{t_0}{2}} (c_\tau + \varphi_N).
    \end{align*}
    By the triangle inequality and elementary computations, we note that (recall that $P = P_{t_0}$),
    \begin{align*}
       \left| P_{\frac{t_0}{2}} \log P_{\frac{t_0}{2}} (c_\tau + \varphi_N) - P_{\frac{t_0}{2}} \log P_{\frac{t_0}{2}} (v_\tau) \right| 
       \leq \frac1{c_\tau} P (|c_\tau +\varphi_N - v_\tau|) \underset{N\to\infty}{\longrightarrow} 0,
    \end{align*}
    from which we obtain \Cref{eq:sgld-proof-crucial-step} for $f = v_\tau$.
    
    Thus, by reversibility of $\pi$ under the semigroup $(P_t)_{t> 0}$, we have
    \begin{align*}
       \frac1{q(t_0)} \Eof[\pi]{v_\tau\log P \left( v_\tau^{q(t_0)} \right)}  \leq  \Eof[\pi]{P_{\frac{t_0}{2}} v_\tau \log P_{\frac{t_0}{2}} v_\tau}.
    \end{align*}
    We now plug this into \Cref{eq:sgld-decay-step-1} and get that
    \begin{align*}
        \frac{\der}{\der \tau} \klb{\rho_\tau^S}{\pi} \leq \frac1{q(t_0)} \Eof[\rho_\tau^S]{\klb{\delta_x P_S}{\delta_x P}} &+ \Eof[\pi]{ P_{\frac{t_0}{2}} v_\tau\log P_{\frac{t_0}{2}} v_\tau - v_\tau P \log(v_\tau)} \\&+ 
        \Eof[\pi]{v_\tau P \log(v_\tau) - v_\tau \log v_\tau},
    \end{align*}
    which by invariance of $\pi$ under $P_{t_0/2}$ leads to
    \begin{align*}
        \frac{\der}{\der \tau} \klb{\rho_\tau^S}{\pi} &\leq \frac1{q(t_0)} \Eof[\rho_\tau^S]{\klb{\delta_x P_S}{\delta_x P}} + \Eof[\pi]{P_{\frac{t_0}{2}} v_\tau \log P_{\frac{t_0}{2}} v_\tau - P_{\frac{t_0}{2}} \left( v_\tau \log(v_\tau) \right)} \\
        &= \frac1{q(t_0)} \Eof[\rho_\tau^S]{\klb{\delta_x P_S}{\delta_x P}} - \left( \ent{v_\tau} - \ent{P_{\frac{t_0}{2}} v_\tau} \right) . 
    \end{align*}
    By the curvature condition, the invariant distribution $\pi$ satisfies the (classical) logarithmic Sobolev inequality with constant $K$. Thus, by the exponential convergence of entropy along the semigroup \citep{bakry_analysis_2014}, we have
    \begin{align*}
         \frac{\der}{\der \tau} \klb{\rho_\tau^S}{\pi} \leq \frac1{q(t_0)} \Eof[\rho_\tau^S]{\klb{\delta_x P_S}{\delta_x P}} - \left( 1 - e^{-K t_0} \right) \klb{\rho_\tau^S}{\pi}.
    \end{align*}
    We conclude by the Grönwall lemma and noting that $\exp \big( \int_{\frac{t_0}{2}}^{t_0} \frac{-2K}{e^{2Ku} - 1} \der u \big) = \frac{1 - e^{-Kt_0}}{1 - e^{-2Kt_0}}$.
\end{proof}

We can now prove \Cref{thm:sgld-kl-bound} as a corollary of the above proposition.


\begin{proof} \textbf{(\Cref{thm:sgld-kl-bound})}
    We apply \Cref{prop:sgld-exponential-decay} by introducing the same SDE and notations as in the proof of \Cref{ex:ornstein-uhlenbeck-example}.
    With these notations, we have
    \begin{align*}
        K := \frac1{\sigma_\pi^2} = \frac{1 - (1 - \lambda \eta)^2}{\sigma^2}, \quad t_0 = -\sigma_\pi^2 \log (1 - \lambda \eta), \quad \frac{1 - e^{-Kt_0}}{1 - e^{-2Kt_0}} = \frac1{2 - \lambda \eta}.
    \end{align*}
    By the joint convexity of the KL divergence, we have the classical bound on local KL divergence at $x \in \Rd$, (see in particular \citep{neu_information-theoretic_2021}). Let $\Xi \sim \mathcal{N}(0,I_d)$ and $U$ be independent of $\Xi$ representing the randomness of the stochastic gradients, as in \Cref{ex:noisy-sgd-noisy-algorithm}.
    \begin{align*}
        \klb{\delta_x P_S}{\delta_x P} &= \klb{\law \left( (1 - \lambda \eta)x - \eta \widehat{g}_S(x, U) + \sigma \Xi \right)}{\law \left( (1 - \lambda \eta)x + \sigma \Xi \right)} \\
        &\leq \frac1{2\sigma^2} \Wrm_2 \left( \law \left( \eta \widehat{g}_S(x, U)\right), \delta_0 \right)^2 \\
        &\leq  \frac{\eta^2}{2\sigma^2} \Eof{\normof{\widehat{g}_S(x,U)}^2}.
    \end{align*}
    Therefore, we obtain that
    \begin{align*}
        \klb{\rho_T^S}{\pi} \leq \frac{\eta^2 (2 - \lambda \eta)}{2  \sigma^2} \int_0^T e^{-\lambda \eta (T - t)}\EofLigne{\normof{\widehat{g}_S(Y_t^S,U)}^2 | S}  \der t.
    \end{align*}
    We conclude the proof by applying \Cref{thm:subgaussian-pac-bayes} as in the proof of \Cref{thm:generalization-under-modified-lsi}.
\end{proof}

The following technical lemma shows that \Cref{ass:compatibility-condition} is satisfied in the setting of \Cref{thm:sgld-kl-bound} under a gradient Lipschitz condition.

\begin{lemma}
    \label{lemma:compatibility-for-sgld}
    Consider the setting of \Cref{sec:sgld} and assume that the stochastic gradients $\widehat{g}_S(\cdot, U)$ are Lipschitz continuous for all $U$ and $S \in \zcal$, with the notations of \Cref{ex:noisy-sgd-noisy-algorithm}.
    We also assume that $\mu_0$ has finite second-order moment, that $\mu_0 \ll \pi$, and that  $\der \mu_0 / \der \pi$ is bounded away from zero.
    Then \Cref{ass:compatibility-condition} is satisfied.
\end{lemma}

\begin{proof}
    Let us first recall some notation. $P_S$ is the Markov operator of the process $X_{k+1}^S = (1 - \lambda \eta) X_k^S - \eta \widehat{g}_S(X_k^S, U_k) + \sigma \xi_k$, with $(\xi_k)_{k \in \N} \sim \mathcal{N}(0, I_d)^{\otimes \infty}$ independent of the $(U_k)_{k \in \N}$. The prior distribution is $\pi := \mathcal{N}(0, \sigma_\pi^2 I_d)$, with $\sigma_\pi^2 := \sigma^2 / (1 - (1 - \lambda \eta)^2) > \sigma^2$.

    \textbf{Step 1:} The continuity preservation condition (\Cref{ass:continuity-preservation}) is clearly satisfied by properties of the Gaussian convolution.

    \textbf{Step 2:} It remains to check \Cref{ass:bounded-entropy}, for a fixed $S\in\zcal^n$. For this we note that $\Phi(x):=x\log(x)$ is lower bounded and $\int \Phi(P_S^\star v_t) \der \pi = \klb{\rho_t^S P_S}{\pi}$, by definition of $P_S^\star$. Moreover, we have $\sigma_\pi^2 > \sigma^2$, so we can write $\rho_t^S P_S = \mu_t \ast \mathcal{N}(0,\sigma^2 I_d)$ and $\pi = \mathcal{N}(0, (\sigma_\pi^2 - \sigma^2)I_d) \ast \mathcal{N}(0,\sigma^2 I_d)$, with $\mu_t$ the law of $Z_t := (1 - \lambda \eta) Y_t^S - \eta \widehat{g}_S(Y_t^S, U_k)$. Therefore, by the joint convexity of the KL divergence, we have
    $$
        \int \Phi(P_S^\star v_t) \der \pi \leq \frac1{2\sigma^2} \Wrm_2 (\mu_t, \mathcal{N}(0, (\sigma_\pi^2 - \sigma^2)I_d))^2 \lesssim \frac1{\sigma^2} \left(\Eof{\normof{Z_t^2}} + d (\sigma_\pi^2 - \sigma^2) \right).
    $$
    Let us write $T(\cdot, U) := (1 - \lambda \eta) \mathrm{Id} - \eta \widehat{g}_S(\cdot, U)$. By our assumption, $T(\cdot, U) : \Rd \to \Rd$ is Lipschitz continuous, let $L$ be its Lipschitz constant. Note that $U$ represents the randomness of the batch indices, so in particular its distribution has finite support. Therefore, we have $\Eof{\normof{T(0,U_k)}^2} \leq C < +\infty $, for some $C > 0$.
    Therefore, we have
    $$
        \EofLigne{\normof{Z_t}^2} \lesssim C_0 + \EofLigne{\normof{Y_t^S}^2} = C_0 + e^{-t} \sum_{k\in \N} \frac{t^k}{k!} \EofLigne{\normof{X_k^S}^2}.
    $$
    Finally, we have $\EofLigne{\normofLigne{X_{k+1}^S}^2} \lesssim d \sigma^2 + C_0 + L^2 \EofLigne{\normofLigne{X_{k}^S}^2}$. As $\int \normofLigne{x}^2 \der\mu_0(x) < \infty$, the growth of $k\mapsto\EofLigne{\normof{X_{k}^S}^2}$ is at most exponential, so the above sum is finite. This concludes the proof.
\end{proof}

\subsection{Omitted proofs of \Cref{sec:sgd-perturbed-last-iterate}}
\label{sec:proofs-sgd-perturbed-last-iterate}

The structure of the perturbed Poissonized process is made precise by the following lemma.
\begin{restatable}{lemma}{lemmaPerturbedPoissonStructure}
\label{lemma:perturbed-poissonized-process-structure}
 $Y_t^S $ and $\xi_{N_t}$ are independent and $\xi_{N_t} \sim \mathcal{N}(0,I_d)$.
\end{restatable}
\begin{proof}
Let $A,B \in \borel$, by the definition of $(\xi_k)_{k\in\N}$, we have, for a fixed $S \in \zcal^n$,
\begin{align*}
    \Pof{Y_t^S \in A, ~ \xi_{N_t} \in B} &= \sum_{k\in\N} e^{-t} \frac{t^k}{k!} \Pof{X_k^S \in A, ~ \xi_{k} \in B} \\
    &= \sum_{k\in\N} e^{-t} \frac{t^k}{k!} \Pof{X_k^S \in A} \Pof{\xi_{0} \in B} \by{independence} .
\end{align*}
Moreover, for all $A \in \borel$, we have $ \Pof{\xi_{N_t} \in A} = \sum_{k=0}^{+\infty} e^{-t} \frac{t^k}{k!}  \Pof{\xi_{k} \in A} =  \Pof{\xi_{0} \in A}$.
This concludes the proof.
\end{proof}
\textbf{Operator setup.} Before presenting the proof of the main results of \Cref{sec:sgd-perturbed-last-iterate}, we define below some Markov operators associated with the stochastic processes defined in \Cref{sec:sgd-perturbed-last-iterate}.
\begin{itemize}[noitemsep,nosep]
    \item $P_S$ the Markov operator associated to the SGD recursion defined in \Cref{sec:sgd-perturbed-last-iterate}.
    \item $Q$ is the heat kernel defined by $Qf(x) := \Eof[\xi \sim \mathcal{N}(0,I_d)]{f(x + \sigma \xi)}$ (when it makes sense).
    \item Let $\Pbar$ be the operator associated with the discrete Ornstein-Uhlenbeck process,
    \begin{align*}
        X_{k+1} = (1 - \lambda \eta) X_k + \sigmabar \xi_k, \quad (\xi_k)_{k\in\N} \sim \mathcal{N}(0,I_d)^{\otimes \infty}, \quad \sigmabar := \sigma \sqrt{1 - (1 - \lambda \eta)^2}.
    \end{align*}
    \item We use the invariant distribution $\pi := \mathcal{N}(0, \sigma^2 I_d)$ of $\Pbar$ as prior distribution.
    \item Let $P := Q \Pbar$, $P$ is the Markov kernel of the discrete Ornstein-Uhlenbeck process,
    \begin{align*}
        X_{k+1} = (1 - \lambda \eta) X_k + \sigma \xi_k, \quad (\xi_k)_{k\in\N} \sim \mathcal{N}(0,I_d)^{\otimes \infty}.
    \end{align*}
\end{itemize}

As before, we denote by $\rho_t^S$ the distribution of $Y_t^S$ and define $v_t := \der \rho_t^S / \der \pi$.

\begin{remark}
    The fact that $P_S$ satisfies the continuity preservation condition is satisfied as soon as $\nabla\ell(\cdot, z)$ is $\beta$-Lipschitz continuous for all $z\in\zcal$ and $\eta < 1/\beta$ \citep{clerico_generalisation_2023}. 
\end{remark}

\textbf{Perturbed Boltzmann equation.} As before, define the Radon-Nikodym derivative, $\vbar_t := {\der \rhobar_t^S} / {\der \pi}$.
We derive below the perturbed Boltzmann equation for $\vbar_t$.

\begin{restatable}[Perturbed Boltzmann equation]{lemma}{lemmaPerturbedBoltzmann}
    \label{lemma:perturbed-boltzmann-equation}
    We have $\partial \vbar_t = ( Q^\star P_S^\star  - Q^\star) v_t$,
    where $Q^\star$ denotes the adjoint of $Q$ with respect to $\pi$, as constructed in \Cref{sec:assumptions-entropy-flow}.
\end{restatable}

\begin{proof}
We start by noting that $ {\der \rhobar_t^S} / {\der x} = Q \left(  {\der \rho_t^S} / {\der x} \right)$,
where $\rho_t^S := \law (Y_t^S)$ and $Qf(x) := \Eof{f(x + \sigma \Xi)}$ for $\Xi \sim \mathcal{N}(0, I_d)$ is the heat kernel with variance $\sigma^2$.
Indeed, by \Cref{lemma:perturbed-poissonized-process-structure}, we have $\rhobar_t^S = \rho_t^S \ast \mathcal{N}(0, \sigma^2 I_d)$, then the result follows from the self-adjointness of $Q$.


In particular, we have that $\rhobar_t^S = \rho_t^S Q$.
Now it is clear that $Q$ satisfies the continuity presrvation condition with respect to $\pi$, therefore, we can define its adjoint $Q^\star$ with respect to $\pi$ in \Cref{eq:l1-adjoint-construction}.
By definition, we have
\begin{align*}
    \vbar_t(x) = Q^\star v_t(x) = \frac1{\pi(x)} \int_\Rd \pi(x - y) v_t(y)  \pi(y) \der y. 
\end{align*}
This leads to the following computation.
\begin{align*}
    \partial_t \vbar_t (x) = \frac1{\pi(x)} \partial_t \int_\Rd \pi(x - y) v_t(y) \pi(y) \der y.
\end{align*}
We note that, for $t\in[a,b]$ with $a < b$, we have $|\partial_t v_t| = |(P_S^\star - I) v_t| \leq e^{b - a} \left( P_S^\star v_b + v_b \right) \in \Lrm^1(\pi)$.
As $\pi$ is bounded, we can differentiate under the integral to get that
\begin{align*}
    \partial_t \vbar_t (x) = \frac1{\pi(x)} \int_\Rd \pi (x - y) (P_S^\star - I) v_t(y) \pi(y) \der y.
\end{align*}
Therefore, we obtain the perturbed Boltzmann equation for $\vbar_t$, $\partial \vbar_t = ( Q^\star P_S^\star  - Q^\star) v_t$.
\end{proof}

This leads to the following perturbed entropy flow formula.

\begin{restatable}[Perturbed entropy flow]{corollary}{corPerturbedEntropyFlow}
    \label{cor:perturbed-entropy-flow}
    Assume that \Cref{ass:perturbed-sgd-assumptions} holds and that $Q^\star P_S^\star$ satisfies condition \ref{ass:bounded-entropy} in \Cref{ass:compatibility-condition}. Then, we have
    \begin{align*}
        \timeder \klb{\rhobar_t^S}{\pi} = \int (P_S Q - P) \log (\vbar_t) v_t \der \pi - \ecal_{\pi,\Pbar} (\log(\vbar_t), \vbar_t).
    \end{align*}
\end{restatable}

\begin{proof}
As in the proof of \Cref{thm:entropy-flow-weak-regularity}, we first verify that we can differentiate under the integral. Under our assumptions, by \Cref{lemma:sgd-stays-bounded-dissipative}, $v_t$ has compact support included in $\overline{B}(0, L/\lambda)$. This observation extends to $P_S^\star v_t$, as its support is that of the distribution $\rho_t^S P_S$. 

Let $a < b$ and $t \in [a,b]$. Observing that $\vbar_t = Q^\star v_t$ (by \Cref{lemma:perturbed-poissonized-process-structure}) and $\vbar_t \geq e^{-b} \vbar_0$, we apply \Cref{lemma:exponential-bound-density} below to show a bound of the form $|\log (\vbar_t) ( Q^\star P_S^\star  - Q^\star) v_t|(x) \lesssim \mathrm{Pol}(\normof{x}) e^{C \normof{x}}$,
with $C>0$ and $\mathrm{Pol}$ constant and polynomial depending only on $(a,b,d,\sigma,\lambda,\eta,L)$. This ensures the uniform integrability of the above quantity with respect to $\pi$ on $[a,b]$. Therefore,
\begin{align*}
    \timeder \klb{\rhobar_t^S}{\pi} &= \int \log (\vbar_t) ( Q^\star P_S^\star  - Q^\star) v_t \der \pi \\
    &= \int (P_S Q - P) \log (\vbar_t) v_t \der \pi - \int \vbar_t \log(\vbar_t) \der \pi + \int v_t Q\Pbar \log(\vbar_t) \der \pi \\
    &= \int (P_S Q - P) \log (\vbar_t) v_t \der \pi - \ecal_{\pi,\Pbar} (\log(\vbar_t), \vbar_t).
\end{align*}
This concludes the proof.
\end{proof}
\begin{lemma}
    \label{lemma:exponential-bound-density}
    Under the assumptions of \Cref{thm:sgd-perturbed-last-iterate}, let $\mu \in \probameasures$ be a Borel probability measure with compact support $K \subset \Rd$. Let $Q(dx | y) := \mathcal{N}(y, \sigma^2 I_d)$ and $\pi :=  \mathcal{N}(0, \sigma^2 I_d)$. Then, there exist $a_1, a_2 > 0$, depending only on $(K, \sigma, d)$ such that $\der (\mu Q) / \der \pi (x) \leq a_1 e^{a_2 \normof{x}}.$
    
    If, moreover, $\mu \ll \pi$ and there exist $\mathcal{O} := B(0,r) \subset K$ with $r > 0$, and $\inf_{\ocal} (\der \mu / \der \pi) > 0$, then there exist $b_1, b_2 > 0$, depending only on $(r, \sigma, d, \mu)$ such that $\vbar(x) \geq b_1 e^{-b_2 \normof{x}}$.
\end{lemma}

\begin{proof}
    \textbf{First part:} Let $\vbar := \der (\mu Q) / \der \pi$. For $x \in \Rd$, we have
    \begin{align*}
        \vbar(x) = \intrd \frac{\pi(x - y)}{\pi(x)} \der \mu(y) = \int_K \exp \left\{ -\frac{1}{2 \sigma^2} \left( \normof{y}^2 - 2 \langle x, y\rangle  \right) \right\} \der \mu(y).
    \end{align*}
    Let $R>0$ such that $K \subset B(0,R)$, then we have $ \vbar(x) \leq e^{R \normof{x} / \sigma^2}$.

    \textbf{Second part:} Let $x \in \Rd$ and $v := \der \mu / \der \pi$, by a similar computation, we have
    \begin{align*}
        \vbar(x) &\geq (2\pi \sigma^2)^{-d / 2} \inf_\ocal (v) \int_\ocal \exp \left\{ -\frac{1}{\sigma^2} \left( \normof{y}^2 -  \langle x, y\rangle  \right) \right\} \der y \geq C e^{- r \normof{x}^2 / \sigma^2},
    \end{align*}
    where $C> 0$ depends only on $(d, \sigma, r)$. This concludes the proof.
\end{proof}

We now show that, under the assumptions of \Cref{thm:sgd-perturbed-last-iterate}, $v_t^S$ has compact support.

\begin{lemma}
    \label{lemma:sgd-stays-bounded-dissipative}
    Consider the SGD recursion as above and assume that \Cref{ass:perturbed-sgd-assumptions} hold. Further assume that $\normof{X_0^S} \leq L / \lambda$.
    Then, for all $S \in \zcal^n$ and all $k \in \N$, we have almost surely (on the noise of the stochastic gradient) that $\normof{X_k^S} \leq L / \lambda$.
\end{lemma}

\begin{proof}
    By the Cauchy-Schwarz and Young inequalities, we have
    \begin{align*}
        \normof{X_{k+1}^S}^2 &= (1 - \lambda \eta)^2 \normof{X_k^S}^2 + \eta^2 \normof{\widehat{g}_S(X_k^S, U_k)}^2 - 2 \eta (1 - \lambda \eta) \langle X_k^S, \widehat{g}_S(X_k^S, U_k) \rangle \\
            &\leq \left( (1 - \lambda \eta)^2 + \eta\lambda (1 - \eta\lambda) \right) \normof{X_k^S}^2 + \big((\eta^2 + \frac{\eta (1 - \eta \lambda)}{\lambda} \big) \normof{\widehat{g}_S(X_k^S, U_k)}^2 \\
            &= (1 - \lambda \eta) \normof{X_k^S}^2 + \frac{\eta L^2}{\lambda}.
    \end{align*}
    As we have $\normof{X_0^S}^2 \leq L^2 / \lambda^2$, we immediately deduce the result by recursion.
\end{proof}

\subsection{Proof of \Cref{thm:sgd-perturbed-last-iterate}}


\begin{proof} \textbf{(of \Cref{thm:sgd-perturbed-last-iterate})}
    Let $S \in \zcal^n$.
    By \Cref{cor:perturbed-entropy-flow}, we have
    \begin{align*}
        \timeder \ent[\pi]{\vbar_t} = \Delta_t - \ecal_{\pi,\Pbar} (\log(\vbar_t), \vbar_t),\quad \Delta_t := \int (P_S Q - P) \log (\vbar_t) v_t \der \pi.
    \end{align*}
    For the second term, we note that $\pi$ has been constructed as the reversible measure of the operator $\Pbar$. Therefore, we can use the modified log-Sobolev inequality for $\ecal_{\pi,\Pbar}$, which gives $\timeder \ent[\pi]{\vbar_t} \leq \Delta_t - \lambda \eta \ent[\pi]{\vbar_t}$.
    For the term $\Delta_t$, we use the proof of \Cref{prop:expansion-term-general-bound-noisy} to get
    \begin{align}
        \label{eq:delta-bound-perturbed-sgd}
        \Delta_t \leq  \int {\inf_{q>0} \left(\frac1{q} \klb{\delta_x (P_SQ)}{\delta_x P} + \int_0^q \entbar_{\delta_x P}(\vbar_t^\epsilon) \frac{\der \epsilon}{\epsilon^2}  \right)} \der \rho_t^S (x).
    \end{align}
    Observe that $P_S Q$ is the Markov operator of $X_{k+1}^S = (1 - \lambda \eta) X_k^S - \eta \widehat{g}_S(X_k, U_k) + \sigma \xi_k$.
    with the notations of \Cref{ex:noisy-sgd-noisy-algorithm}.
    By the joint convexity of the KL divergence, we obtain
    \begin{align*}
        \klb{\delta_x (P_SQ)}{\delta_x P} \leq \frac{\eta^2}{2\sigma^2} \EofLigne{\normof{\widehat{g}_S (x, U)}^2 | S},
    \end{align*}
    where $U \sim \mu_U$ is the internal randomness of $\widehat{g}_S$ and is independent of $S$.
    For the integral term, we use the (classical) logarithmic Sobolev inequality satisfied by $\delta_x P = \mathcal{N}((1 - \lambda \eta)x, \sigma^2 I_d)$,
    (note that, by convolution, $\vbar_t$ is continuously differentiable). This gives
    \begin{align*}
         \entbar_{P}(\vbar_t^\epsilon) \leq \frac{\sigma^2}{2 P(\vbar_t^\epsilon)} P \big( {\normof{\vbar_t^\epsilon}^2} / {\vbar_t^\epsilon}  \big) = \frac{\sigma^2\epsilon^2 }{2 P(\vbar_t^\epsilon)} P \left( \vbar_t^\epsilon \normof{\nabla \log \vbar_t}^2 \right),\quad \epsilon>0.
    \end{align*}
    It remains to estimate the score of the relative density $\vbar_t$, defined for all $y \in \Rd$ by
    \begin{align*}
        \nabla \log \vbar_t (y) = \nabla \log  \frac{\der \rhobar_t^S}{\der x} (y) - \nabla \log \pi(y) = \nabla \log  \frac{\der \rhobar_t^S}{\der x} (y) + \frac{y}{\sigma^2}.
    \end{align*}
    Recall that $\rhobar_t^S$ is the probability distribution of $\Ybar_t^S = Y_t^S + \sigma \xi_{N_t}$. We can apply the Fisher's identity \citep{efron2011tweedie}, we have that, for all $y\in\Rd$ and fixed $S \in \zcal^n$ and $t > 0$,
    \begin{align*}
         \nabla \log  \frac{\der \rhobar_t^S}{\der x} (y) = \nabla \log p_{\Ybar_t^S} (y) = \EofLigne{\nabla \log p_{\Ybar_t^S | Y_t^S} (y | Y_t^S) \big| \Ybar_t^S = y}.
    \end{align*}
    The conditional density of $\Ybar_t^S$ given $Y_t^S$ is given by \Cref{lemma:perturbed-poissonized-process-structure}. Thus,
    \begin{align*}
        \nabla \log \vbar_t (y) = -\frac1{\sigma^2}\EofLigne{{y - Y_t^S} \big| \Ybar_t^S = y} + \frac{y}{\sigma^2} = \frac1{\sigma^2}\EofLigne{Y_t^S \big| \Ybar_t^S = y}.
    \end{align*}
    By \Cref{lemma:sgd-stays-bounded-dissipative} below, we have that, almost surely, $\normofLigne{X_k^S} \leq L / \lambda$ for all $k \in \N$. We deduce that $\normofLigne{Y_t^S} \leq L / \lambda$. Therefore, by Jensen's inequality, we have $\normofLigne{\nabla \log \vbar_t(y)}^2 \leq \frac{L^2}{\lambda^2 \sigma^4}$.
    Thus,
    \begin{align*}
        \int_0^q \entbar_{\delta_x P}(\vbar_t^\epsilon) \frac{\der \epsilon}{\epsilon^2} \leq \frac{q L^2}{2 \sigma^2 \lambda^2}.
    \end{align*}
    We conclude by optimizing over $q$ and applying Grönwall's lemma like in \Cref{thm:generalization-under-modified-lsi}.
\end{proof}

\subsection{Omitted proofs of \Cref{thm:noise-injection-strongly-convex}}


\begin{proof} \textbf{(of \Cref{thm:noise-injection-strongly-convex})}
    Let $P$ be the Markov operator associated with the Markov chain $X_{k+1} = (1 - \eta \gamma) X_k - \eta \gamma \sigma \mathcal{N}(0, I_d)$. Its reversible distribution is $\pi$ defined above.
    By \Cref{thm:subgaussian-pac-bayes} and \Cref{prop:sgld-exponential-decay} (we postpone the verification of \Cref{ass:compatibility-condition} to the end of the proof), we have that with probability at least $1 - \zeta$ over $S\sim\datadist$, for all $T>0$,
    \begin{align*}
         \Eof{G_S(Y_t^S) | S} \leq \frac{2 \Sigma}{\sqrt{n}} \left\{ (2 - \eta \gamma)\int_0^T e^{-\gamma \eta(T - t)} \Eof[x \sim \rho_t^S]{\klb{\delta_x P_S}{\delta_x P}} + K_T + \log \frac{3}{\zeta} \right\}^{1 / 2}.
    \end{align*}
    For all $x\in\Rd$, we see that
     \begin{align*}
         \klb{\delta_x P_S}{\delta_x P} = \klb{\law (\eta \nabla \er (x + \sigma \Xi))}{\law (\eta \gamma x +  \eta \gamma \sigma \Xi)}, \quad \Xi \sim \mathcal{N}(0, I_d).
     \end{align*}
     By reparameterization of the KL divergence, we have $\klb{\delta_x P_S}{\delta_x P} = \klb{T_\#\mu}{\mu}$, where $\mu := \mathcal{N}(x, \sigma^2 I_d)$ and $T := \gamma^{-1} \nabla \er$.
     As $\er$ is strongly-convex, we have $\normofLigne{T(x) - T(y)} \geq \normofLigne{x - y}$. Thus, $T: \Rd \to T(\Rd)$ is a $\mathcal{C}^1$-diffeomorphism. Denoting by $h$ the lebesgue density of $\mu$, the change of variable formula gives that the Lebesgue density of $T_\#\mu$ is $h \circ T^{-1} |\det (\nabla T) \circ T^{-1}|^{-1}$, where $\nabla T$ is the Jacobian matrix of $T$. Therefore,
     \begin{align*}
         \klb{\delta_x P_S}{\delta_x P} &= \int \log \bigg( \frac{h(y) |\det (\nabla T (y))|^{-1}}{h (T(y))} \bigg) h(y) \der y \\
        &= \int \bigg( - \log \det (\nabla T (y))  +  \frac{\normof{x - T(y)}^2 - \normof{x - y}^2 }{2 \sigma^2}\bigg) h(y) \der y .
     \end{align*}
     By Jacobi's formula, we have that $\log \det (\nabla T (y)) = \trace \log (\nabla T (y)) \geq 0 $, where we used that $\gamma^{-1} \er$ is $1$-strongly convex. Thus, noting that $\nabla \el = T - \mathrm{Id}$ ,we have
     \begin{align*}
          \klb{\delta_x P_S}{\delta_x P} &\leq  \int \frac{\normof{x - T(y)}^2 - \normof{x - y}^2 }{2 \sigma^2} h(y) \der y \\
          &=  \int \frac{\normof{y - T(y)}^2 + 2 \langle T(y) - y, y - x \rangle}{2 \sigma^2} h(y) \der y  \\
          &= \frac1{2\sigma^2} \int \normof{\nabla \el}^2 \der\mu - \int \langle \nabla \el (y), \nabla \log h(y) \rangle h(y) \der y \\
          &= \frac1{2\sigma^2} \int \normof{\nabla \el}^2 \der\mu + \int \trace (\nabla^2 \el) \der \mu , 
     \end{align*}
     where the last line follows by integration by parts.
     
    \textbf{Verification of \Cref{ass:compatibility-condition}:} We note that $\int \Phi(P_S^\star v_t) \der \pi = \klb{\rho_t^S P_S}{\pi} = \klb{\rho_t^S P_S}{\pi P}$, as $\pi$ is $P$-invariant. 
    By the chain rule and joint convexity of KL divergence,
    \begin{align*}
        \klb{\rho_t^S P_S}{\pi} &= \klb{\rho_t^S}{\pi} + \Eof[\rho_t^S]{\klb{\delta_x P_S}{\delta_x P}} \\
        &\leq e^{-t} \sum_{k \in \N} \frac{t^k}{k!} \klb{\mu_0 P_S^k}{\pi P^k} + \Eof[\rho_t^S]{\klb{\delta_x P_S}{\delta_x P}}.
    \end{align*}
    Similarly, $\klb{\mu_0 P_S^k}{\pi P^k} \leq \klb{\mu_0 P_S^k}{\pi P^k} + \Eof[\mu_0 P_S^{k-1}]{\klb{\delta_x P_S}{\delta_x P}}$. By the smoothness property and the above computation, we see that $\klb{\delta_x P_S}{\delta_x P}$ is uniformly bounded and we have a polynomial (in $k$) bound on $\klb{\mu_0 P_S^k}{\pi P^k}$, which concludes the proof.
\end{proof}

\subsection{Omitted proofs of \Cref{thm:sgd-application-wasserstein}}
\label{sec:proofs-sgd-linear-growth}


\begin{proof} (of \Cref{thm:sgd-application-wasserstein})
    Let $P$ be the  Markov operator of the Markov chain $X_{k+1} = (1 - \lambda \eta) X_k + \sigmabar \mathcal{N}(0, I_d)$, with $\sigmabar := 
    \sigma \sqrt{1 - (1 - \lambda \eta)^2}$. Its invariant distribution is $\pi := \mathcal{N}(0, \sigma^2 I_d)$. By \Cref{lemma:lsi-transfer-general-lemma} and \Cref{ex:ornstein-uhlenbeck-example}, $\ecal_{\pi,P}$ satisfies the modified LSI with constant $\lambda \eta$.
    Therefore, by \Cref{thm:generalization-under-modified-lsi}, we have, with probability at least $1 - \zeta$ over $S \sim \datadist$, for all $T>0$,
    \begin{align*}
        \Eof[\rho_T^S]{G_S} \leq \frac{2\Sigma}{\sqrt{n}} \left\{ \int_{0}^T e^{-\lambda \eta (T - t) } \Eof[\rho_t^S]{(P_S - P)(\log v_t)} \der t + e^{-\lambda \eta T} \klb{\mu}{\pi} + \log(3 / \zeta) \right\}^{1 / 2}.
    \end{align*}
    Let $v_t := \der \rho_t^S / \der \pi$.
    By assumption, we have $\normof{\nabla \log v_t} = \normof{\nabla \log u_t^S (x) + x / \sigma^2} \leq a \normof{x} + b$.
    Therefore, by \Cref{prop:first-term-bound-f-regular}, we have that
    \begin{align*}
        \Eof[\rho_t^S]{(P_S - P)(\log v_t)} \leq \Eof[x \sim \rho_t^S]{\Wrm_2(\delta_x P, \delta_x P_S)^2}^{1 / 2} \big( \frac{a}{2} \normof{P}_t + \frac{a}{2} \normof{P_S}_t + b \big),
    \end{align*}
    with the notations of \Cref{prop:first-term-bound-f-regular}.
    Then, we compute separately,
    \begin{align*}
        &\Wrm_2(\delta_x P, \delta_x P_S)^2 \leq \eta^2 \mathds{E}_U [{\normof{\widehat{g}_S(x, U)}^2}] + \sigmabar^2 d,\quad \text{and}\\
        &\normof{P}_t^2 = \EofLigne{(1 - \lambda \eta)^2 \normof{Y_t^S}^2 | S} + \sigmabar^2 d, \quad \normof{P_S}_t^2 \leq 2\EofLigne{(1 - \lambda \eta)^2 \normof{Y_t^S}^2 + \eta^2\normof{\widehat{g}_S(x, U)}^2 | S}.
    \end{align*}
    The result follows by noting that $\sigmabar^2 \leq 2 \sigma^2 \lambda \eta$ and $(1 - \lambda \eta)^2 < 1$.
\end{proof}

\bibliography{main}

@article{rudolf_explicit_2012,
  title = {Explicit Error Bounds for {{Markov}} Chain {{Monte Carlo}}},
  author = {Rudolf, Daniel},
  year = {2012},
  journal = {Dissertationes Mathematicae},
  volume = {485},
  eprint = {1108.3201},
  primaryclass = {math},
  pages = {1--93}
}

@inproceedings{anastasiou2019normalapproximationstochasticgradient,
      title={Normal Approximation for Stochastic Gradient Descent via Non-Asymptotic Rates of Martingale CLT}, 
      author={Andreas Anastasiou and Krishnakumar Balasubramanian and Murat A. Erdogdu},
      booktitle = 	 {Proceedings of the Thirty-Second Conference on Learning Theory},
      pages = 	 {115--137},
      year = 	 {2019},
      volume = 	 {99},
}

@inproceedings{mandt_variational_2016,
  title = {A {{Variational Analysis}} of {{Stochastic Gradient Algorithms}}},
  booktitle = {International {{Conference}} on {{Machine Learning}} ({{ICML}} 2016)},
  author = {Mandt, Stephan and Hoffman, Matthew D. and Blei, David M.},
  year = {2016}
}

@misc{gallegos_herrada_equivalences_2023,
      title={Equivalences of Geometric Ergodicity of Markov Chains}, 
      author={M. A. Gallegos-Herrada and D. Ledvinka and J. S. Rosenthal},
      year={2023},
      eprint={2203.04395},
      archivePrefix={arXiv},
      primaryClass={math.PR}
}

@book{meyn_markov_1993,
  title = {Markov {{Chains}} and {{Stochastic Stability}}},
  author = {Meyn, Sean P. and Tweedie, Richard L.},
  year = {1993},
  publisher = {Springer}
}

@article{alquier2024user,
  title={{User-friendly Introduction to {PAC-Bayes} Bounds}},
  author={Alquier, Pierre},
  journal={Foundations and Trends{\textregistered} in Machine Learning},
  year={2024},
}

@article{ane_logarithmic_2000,
  title = {On Logarithmic {{Sobolev}} Inequalities for Continuous Time Random Walks on Graphs},
  author = {An{\'e}, C{\'e}cile and Ledoux, Michel},
  year = {2000},
  journal = {Probability Theory and Related Fields},
  volume = {116},
  pages = {573--602}
}

@book{bakry_analysis_2014,
  title = {Analysis and {{Geometry}} of {{Markov Diffusion Operators}}},
  author = {Bakry, Dominique and Gentil, Ivan and Ledoux, Michel},
  year = {2014},
  publisher = {Springer}
}

@misc{xie2021diffusiontheorydeeplearning,
      title={A Diffusion Theory For Deep Learning Dynamics: Stochastic Gradient Descent Exponentially Favors Flat Minima}, 
      author={Zeke Xie and Issei Sato and Masashi Sugiyama},
      year={2021},
      eprint={2002.03495},
      archivePrefix={arXiv},
      primaryClass={cs.LG}
}

@misc{wojtowytsch2021stochasticgradientdescentnoise,
      title={Stochastic gradient descent with noise of machine learning type. Part I: Discrete time analysis}, 
      author={Stephan Wojtowytsch},
      year={2021},
      eprint={2105.01650},
      archivePrefix={arXiv},
      primaryClass={stat.ML}
}

@article{bartlett2002rademacher,
  author       = {Peter Bartlett and
                  Shahar Mendelson},
  title        = {Rademacher and Gaussian Complexities: Risk Bounds and Structural Results},
  journal      = {Journal of Machine Learning Research},
  year         = {2002},
}

@misc{bassily_stability_2020,
  title = {Stability of {{Stochastic Gradient Descent}} on {{Nonsmooth Convex Losses}}},
  author = {Bassily, Raef and Feldman, Vitaly and Guzm{\'a}n, Crist{\'o}bal and Talwar, Kunal},
  year = {2020},
  month = jun,
  number = {arXiv:2006.06914},
  eprint = {2006.06914},
  primaryclass = {cs, math, stat},
  publisher = {arXiv},
  doi = {10.48550/arXiv.2006.06914},
  urldate = {2023-07-15},
  archiveprefix = {arXiv}
}

@article{bobkov_modified_1998,
  title = {On {{Modified Logarithmic Sobolev Inequalities}} for {{Bernoulli}} and {{Poisson Measures}}},
  author = {Bobkov, S. G and Ledoux, M},
  year = {1998},
  month = jul,
  journal = {Journal of Functional Analysis},
  volume = {156},
  number = {2},
  pages = {347--365}
}

@article{bobkov_modified_2006,
  title = {Modified {{Logarithmic Sobolev Inequalities}} in {{Discrete Settings}}},
  author = {Bobkov, Sergey G. and Tetali, Prasad},
  year = {2006},
  month = jun,
  journal = {Journal of Theoretical Probability},
  volume = {19},
  number = {2},
  pages = {289--336}
}

@book{bogachev_measure_2007,
  title = {Measure Theory},
  author = {Bogachev, Vladimir I.},
  year = {2007},
  volume = {Volume 1},
  publisher = {Springer}
}

@book{bottcher_levy_2013,
  title = {L{\'e}vy {{Matters III}}: {{L{\'e}vy-Type Processes}}: {{Construction}}, {{Approximation}} and {{Sample Path Properties}}},
  shorttitle = {L{\'e}vy {{Matters III}}},
  author = {B{\"o}ttcher, Bj{\"o}rn and Schilling, Ren{\'e} and Wang, Jian},
  year = {2013},
  series = {Lecture {{Notes}} in {{Mathematics}}},
  volume = {2099},
  publisher = {Springer International Publishing},
  address = {Cham},
  doi = {10.1007/978-3-319-02684-8},
  urldate = {2023-09-13},
  isbn = {978-3-319-02683-1 978-3-319-02684-8},
  langid = {english}
}

@article{bousquet_stability_2002,
  title = {Stability and Generalization},
  author = {Bousquet, Olivier},
  year = {2002},
  journal = {Journal of Machine Learning Research},
  pages = {499--526}
}

@article{bu_tightening_2020-1,
  title = {Tightening {{Mutual Information Based Bounds}} on {{Generalization Error}}},
  author = {Bu, Yuheng and Zou, Shaofeng and Veeravalli, Venugopal V.},
  year = {2020},
  month = may,
  journal = {IEEE Journal on Selected Areas in Information Theory},
  volume = {1},
  number = {1},
  eprint = {1901.04609},
  primaryclass = {cs, stat},
  pages = {121--130}
}

@article{camuto_fractal_2021,
  title = {Fractal {{Structure}} and {{Generalization Properties}} of {{Stochastic Optimization Algorithms}}},
  author = {Camuto, Alexander and Deligiannidis, George and Erdogdu, Murat A. and G{\"u}rb{\"u}zbalaban, Mert and {\c S}im{\c s}ekli, Umut and Zhu, Lingjiong},
  year = {2021},
  month = jun,
  journal = {Advances in Neural Information Processing Systems 34 (NeurIPS 2021)},
  eprint = {2106.04881},
  primaryclass = {cs, stat},
  urldate = {2022-12-14},
  archiveprefix = {arXiv}
}

@misc{caputo_entropy_2024,
  title = {Entropy {{Contractions}} in {{Markov Chains}}: {{Half-Step}}, {{Full-Step}} and {{Continuous-Time}}},
  shorttitle = {Entropy {{Contractions}} in {{Markov Chains}}},
  author = {Caputo, Pietro and Chen, Zongchen and Gu, Yuzhou and Polyanskiy, Yury},
  year = {2024},
  month = sep,
  number = {arXiv:2409.07689},
  eprint = {2409.07689},
  primaryclass = {math},
  publisher = {arXiv},
  doi = {10.48550/arXiv.2409.07689},
  urldate = {2024-12-30},
  archiveprefix = {arXiv}
}

@article{catoni_pac-bayesian_2007,
  title = {Pac-{{Bayesian Supervised Classification}}: {{The Thermodynamics}} of {{Statistical Learning}}},
  shorttitle = {Pac-{{Bayesian Supervised Classification}}},
  author = {Catoni, Olivier},
  year = {2007},
  journal = {IMS Lecture Notes Monograph Series},
  volume = {56},
  eprint = {0712.0248},
  primaryclass = {stat},
  pages = {1--163}
}

@misc{chafai_logarithmic_2017,
  title = {Logarithmic Sobolev Inequalities Essentials},
  author = {Chafai, Djalil and Lehec, Joseph},
  year = {2017}
}

@article{chen_logarithmic_2008,
  title = {The Logarithmic {{Sobolev}} Constant of Some Finite {{Markov}} Chains},
  author = {Chen, Guan-Yu and Liu, Wai-Wai and {Saloff-Coste}, Laurent},
  year = {2008},
  journal = {Annales de la Facult{\'e} des sciences de Toulouse : Math{\'e}matiques},
  volume = {17},
  number = {2},
  pages = {239--290}
}

@misc{chourasia_differential_2022,
  title = {Differential {{Privacy Dynamics}} of {{Langevin Diffusion}} and {{Noisy Gradient Descent}}},
  author = {Chourasia, Rishav and Ye, Jiayuan and Shokri, Reza},
  year = {2022},
  month = sep,
  number = {arXiv:2102.05855},
  eprint = {2102.05855},
  primaryclass = {cs, stat},
  publisher = {arXiv},
  doi = {10.48550/arXiv.2102.05855},
  urldate = {2024-02-07},
  archiveprefix = {arXiv}
}

@misc{clerico_generalisation_2023,
  title = {Generalisation under Gradient Descent via Deterministic {{PAC-Bayes}}},
  author = {Clerico, Eugenio and Farghly, Tyler and Deligiannidis, George and Guedj, Benjamin and Doucet, Arnaud},
  year = {2023},
  month = apr,
  number = {arXiv:2209.02525},
  eprint = {2209.02525},
  primaryclass = {cs, stat},
  publisher = {arXiv},
  doi = {10.48550/arXiv.2209.02525},
  urldate = {2023-08-28},
  archiveprefix = {arXiv}
}

@article{del_moral_contraction_2003,
  title = {On Contraction Properties of {{Markov}} Kernels},
  author = {Del Moral, P. and Ledoux, M. and Miclo, L.},
  year = {2003},
  month = jun,
  journal = {Probability Theory and Related Fields},
  volume = {126},
  number = {3},
  pages = {395--420}
}

@article{diaconis_logarithmic_1996,
  title = {Logarithmic Sobolev Inequalities for Finite Markov Chains},
  author = {Diaconis, P. and {Saloff-Coste}, L.},
  year = {1996},
  journal = {The Annals of Applied Probability},
  volume = {6},
  number = {3},
  pages = {695--750}
}

@misc{dieuleveut_bridging_2018,
  title = {Bridging the {{Gap}} between {{Constant Step Size Stochastic Gradient Descent}} and {{Markov Chains}}},
  author = {Dieuleveut, Aymeric and Durmus, Alain and Bach, Francis},
  year = {2018},
  month = apr,
  number = {arXiv:1707.06386},
  eprint = {1707.06386},
  primaryclass = {math, stat},
  publisher = {arXiv},
  doi = {10.48550/arXiv.1707.06386},
  urldate = {2024-09-04},
  archiveprefix = {arXiv}
}

@article{dobrushin_central_1956,
  title = {Central {{Limit Theorem}} for {{Nonstationary Markov Chains}}. {{I}}},
  author = {Dobrushin, R. L.},
  year = {1956},
  month = jan,
  journal = {Theory of Probability \& Its Applications},
  volume = {1},
  number = {1},
  pages = {65--80},
  publisher = {{Society for Industrial and Applied Mathematics}},
  issn = {0040-585X},
  doi = {10.1137/1101006},
  urldate = {2024-12-27}
}

@inproceedings{dupuis_generalization_2024,
  title = {Generalization {{Bounds}} for {{Heavy-Tailed SDEs}} through the {{Fractional Fokker-Planck Equation}}},
  booktitle = {Proceedings of the 41st {{International Conference}} on {{Machine Learning}}},
  author = {Dupuis, Benjamin and Simsekli, Umut},
  year = {2024},
  month = jul,
  pages = {12087--12137},
  publisher = {PMLR},
  issn = {2640-3498},
  urldate = {2025-01-08},
  copyright = {All rights reserved},
  langid = {english}
}

@article{dupuis_uniform_2024,
  author  = {Benjamin Dupuis and Paul Viallard and George Deligiannidis and Umut Simsekli},
  title   = {Uniform Generalization Bounds on Data-Dependent Hypothesis Sets via PAC-Bayesian Theory on Random Sets},
  journal = {Journal of Machine Learning Research},
  year    = {2024},
  volume  = {25},
  number  = {409},
  pages   = {1--55}
}

@misc{even_continuized_2021,
  title = {A {{Continuized View}} on {{Nesterov Acceleration}} for {{Stochastic Gradient Descent}} and {{Randomized Gossip}}},
  author = {Even, Mathieu and Berthier, Rapha{\"e}l and Bach, Francis and Flammarion, Nicolas and Gaillard, Pierre and Hendrikx, Hadrien and Massouli{\'e}, Laurent and Taylor, Adrien},
  year = {2021},
  month = jun,
  journal = {arXiv.org},
  urldate = {2024-08-27}
}

@inproceedings{farghly_time-independent_2021,
  title = {Time-Independent {{Generalization Bounds}} for {{SGLD}} in {{Non-convex Settings}}},
  booktitle = {35th {{Conference}} on {{Neural Information Processing Systems}} ({{NeurIPS}} 2021).},
  author = {Farghly, Tyler and Rebeschini, Patrick},
  year = {2021},
  month = nov,
  eprint = {2111.12876},
  primaryclass = {cs, math, stat},
  publisher = {arXiv}
}

@misc{feldman_high_2019,
  title = {High Probability Generalization Bounds for Uniformly Stable Algorithms with Nearly Optimal Rate},
  author = {Feldman, Vitaly and Vondrak, Jan},
  year = {2019},
  month = jun,
  number = {arXiv:1902.10710},
  eprint = {1902.10710},
  primaryclass = {cs, stat},
  publisher = {arXiv},
  doi = {10.48550/arXiv.1902.10710},
  urldate = {2023-08-30},
  archiveprefix = {arXiv}
}

@inproceedings{futami_time-independent_2023,
  title = {Time-{{Independent Information-Theoretic Generalization Bounds}} for {{SGLD}}},
  booktitle = {7th {{Conference}} on {{Neural Information Processing Systems}} ({{NeurIPS}} 2023).},
  author = {Futami, Futoshi and Fujisawa, Masahiro},
  year = {2023},
  month = nov,
  eprint = {2311.01046},
  primaryclass = {cs, stat},
  publisher = {arXiv}
}

@inproceedings{germain_pac-bayesian_2009,
  title = {{{PAC-Bayesian}} Learning of Linear Classifiers},
  booktitle = {Proceedings of the 26th {{Annual International Conference}} on {{Machine Learning}}},
  author = {Germain, Pascal and Lacasse, Alexandre and Laviolette, Fran{\c c}ois and Marchand, Mario},
  year = {2009},
  month = jun,
  series = {{{ICML}} '09},
  pages = {353--360},
  publisher = {Association for Computing Machinery},
  address = {New York, NY, USA}
}

@article{goel_modified_2004,
  title = {Modified Logarithmic {{Sobolev}} Inequalities for Some Models of Random Walk},
  author = {Goel, Sharad},
  year = {2004},
  month = nov,
  journal = {Stochastic Processes and their Applications},
  volume = {114},
  number = {1},
  pages = {51--79}
}

@article{gross_logarithmic_1975-1,
  title = {Logarithmic {{Sobolev Inequalities}}},
  author = {Gross, Leonard},
  year = {1975},
  journal = {American Journal of Mathematics},
  volume = {97},
  number = {4},
  eprint = {2373688},
  eprinttype = {jstor},
  pages = {1061--1083},
  publisher = {Johns Hopkins University Press}
}

@article{efron2011tweedie,
 author = {Bradley Efron},
 journal = {Journal of the American Statistical Association},
 number = {496},
 pages = {1602--1614},
 publisher = {[American Statistical Association, Taylor & Francis, Ltd.]},
 title = {Tweedie's Formula and Selection Bias},
 urldate = {2025-05-13},
 volume = {106},
 year = {2011}
}

@article{Raginsky2014StrongDP,
  title={Strong Data Processing Inequalities and Phi-Sobolev Inequalities for Discrete Channels},
  author={Maxim Raginsky},
  journal={IEEE Transactions on Information Theory},
  year={2014},
  volume={62},
  pages={3355-3389},
}

@article{blanca_entropy_2022,
author = {Antonio Blanca and Pietro Caputo and Daniel Parisi and Alistair Sinclair and Eric Vigoda},
title = {{Entropy decay in the Swendsen–Wang dynamics on ${\mathbb{Z}^{d}}$}},
volume = {32},
journal = {The Annals of Applied Probability},
number = {2},
publisher = {Institute of Mathematical Statistics},
pages = {1018 -- 1057},
year = {2022}
}

@inproceedings{gurbuzbalaban2020heavy,
  author    = {Mert G\"{u}rb\"{u}zbalaban and 
               Umut {\c S}im{\c s}ekli and 
               Lingjiong Zhu},
  title     = {{The heavy-tail phenomenon in {SGD}}},
  booktitle = {International Conference on Machine Learning (ICML)},
  year      = {2021},
}

@misc{haghifam_sharpened_2020,
  title = {Sharpened {{Generalization Bounds}} Based on {{Conditional Mutual Information}} and an {{Application}} to {{Noisy}}, {{Iterative Algorithms}}},
  author = {Haghifam, Mahdi and Negrea, Jeffrey and Khisti, Ashish and Roy, Daniel M. and Dziugaite, Gintare Karolina},
  year = {2020},
  month = oct,
  number = {arXiv:2004.12983},
  eprint = {2004.12983},
  primaryclass = {cs, math, stat},
  publisher = {arXiv},
  doi = {10.48550/arXiv.2004.12983},
  urldate = {2023-09-01},
  archiveprefix = {arXiv}
}

@misc{hardt_train_2016,
  title = {Train Faster, Generalize Better: {{Stability}} of Stochastic Gradient Descent},
  shorttitle = {Train Faster, Generalize Better},
  author = {Hardt, Moritz and Recht, Benjamin and Singer, Yoram},
  year = {2016},
  month = feb,
  number = {arXiv:1509.01240},
  eprint = {1509.01240},
  primaryclass = {cs, math, stat},
  publisher = {arXiv},
  doi = {10.48550/arXiv.1509.01240},
  urldate = {2024-02-22},
  archiveprefix = {arXiv}
}

@article{jacquet_analytical_1998,
  title = {Analytical Depoissonization and Its Applications},
  author = {Jacquet, Philippe and Szpankowski, Wojciech},
  year = {1998},
  month = jul,
  journal = {Theoretical Computer Science},
  volume = {201},
  number = {1},
  pages = {1--62}
}

@misc{jastrzebski_three_2018-1,
  title = {Three {{Factors Influencing Minima}} in {{SGD}}},
  author = {Jastrzebski, Stanis{\l}aw and Kenton, Zachary and Arpit, Devansh and Ballas, Nicolas and Fischer, Asja and Bengio, Yoshua and Storkey, Amos},
  year = {2018},
  month = sep,
  number = {arXiv:1711.04623},
  eprint = {1711.04623},
  primaryclass = {cs, stat},
  publisher = {arXiv},
  doi = {10.48550/arXiv.1711.04623},
  urldate = {2023-01-13},
  archiveprefix = {arXiv},
}

@misc{keskar_large-batch_2017,
  title = {On {{Large-Batch Training}} for {{Deep Learning}}: {{Generalization Gap}} and {{Sharp Minima}}},
  shorttitle = {On {{Large-Batch Training}} for {{Deep Learning}}},
  author = {Keskar, Nitish Shirish and Mudigere, Dheevatsa and Nocedal, Jorge and Smelyanskiy, Mikhail and Tang, Ping Tak Peter},
  year = {2017},
  month = feb,
  number = {arXiv:1609.04836},
  eprint = {1609.04836},
  primaryclass = {cs, math},
  publisher = {arXiv},
  doi = {10.48550/arXiv.1609.04836},
  urldate = {2024-06-25},
  archiveprefix = {arXiv},
}

@book{lasota_chaos_1994,
  title = {Chaos, {{Fractals}} and {{Noise}}},
  author = {Lasota, Andrzej and Mackey, Michael C.},
  year = {1994},
  edition = {Applied Mathematical Sciences 97},
  publisher = {Springer}
}

@book{levin_markov_2017,
  title = {Markov {{Chains}} and {{Mixing Times}}},
  author = {Levin, David A. and Peres, Yuval},
  year = {2017},
  publisher = {American Mathematical Society}
}

@inproceedings{li_generalization_2020,
  title = {On {{Generalization Error Bounds}} of {{Noisy Gradient Methods}} for {{Non-Convex Learning}}},
  booktitle = {Published as a Conference Paper at {{ICLR}} 2020},
  author = {Li, Jian and Luo, Xuanyuan and Qiao, Mingda},
  year = {2020},
  month = feb,
  eprint = {1902.00621},
  primaryclass = {cs, stat}
}

@misc{maurer_note_2004,
  title = {A {{Note}} on the {{PAC Bayesian Theorem}}},
  author = {Maurer, Andreas},
  year = {2004},
  month = nov,
  number = {arXiv:cs/0411099},
  eprint = {cs/0411099},
  publisher = {arXiv},
  doi = {10.48550/arXiv.cs/0411099},
  urldate = {2024-01-14},
  archiveprefix = {arXiv}
}

@article{mcallester_pac-bayesian_1999,
  title = {Some {{PAC-Bayesian}} Theorem},
  author = {McAllester, David},
  year = {1999},
  journal = {Machine Learning}
}

@article{mcallester_pac-bayesian_2003,
  title = {{{PAC-Bayesian Stochastic Model Selection}}},
  author = {McAllester, David A.},
  year = {2003},
  month = apr,
  journal = {Machine Learning},
  volume = {51},
  number = {1},
  pages = {5--21}
}

@inproceedings{mou_generalization_2017,
  title = {Generalization {{Bounds}} of {{SGLD}} for {{Non-convex Learning}}: {{Two Theoretical Viewpoints}}},
  shorttitle = {Generalization {{Bounds}} of {{SGLD}} for {{Non-convex Learning}}},
  booktitle = {Proceedings of the 31st {{Conference On Learning Theory}}},
  author = {Mou, Wenlong and Wang, Liwei and Zhai, Xiyu and Zheng, Kai},
  year = {2017}
}

@inproceedings{negrea_information-theoretic_2020,
 author = {Negrea, Jeffrey and Haghifam, Mahdi and Dziugaite, Gintare Karolina and Khisti, Ashish and Roy, Daniel M},
 booktitle = {Advances in Neural Information Processing Systems},
 pages = {},
 publisher = {Curran Associates, Inc.},
 title = {Information-Theoretic Generalization Bounds for {SGLD} via Data-Dependent Estimates},
 volume = {32},
 year = {2019}
}

@misc{neu_information-theoretic_2021,
  title = {Information-{{Theoretic Generalization Bounds}} for {{Stochastic Gradient Descent}}},
  author = {Neu, Gergely and Dziugaite, Gintare Karolina and Haghifam, Mahdi and Roy, Daniel M.},
  year = {2021},
  month = aug,
  number = {arXiv:2102.00931},
  eprint = {2102.00931},
  primaryclass = {cs, stat},
  publisher = {arXiv},
  doi = {10.48550/arXiv.2102.00931},
  urldate = {2023-01-18},
  archiveprefix = {arXiv}
}

@inproceedings{ben-arous_high_dimensional_2022,
  title = {High-Dimensional Limit Theorems for {{SGD}}: {{Effective}} Dynamics and Critical Scaling},
  booktitle = {Advances in Neural Information Processing Systems},
  author = {Ben Arous, Gerard and Gheissari, Reza and Jagannath, Aukosh},
  editor = {Koyejo, S. and Mohamed, S. and Agarwal, A. and Belgrave, D. and Cho, K. and Oh, A.},
  year = {2022},
  volume = {35},
  pages = {25349--25362},
  publisher = {Curran Associates, Inc.}
}

@misc{neyshabur_exploring_2017,
  title = {Exploring {{Generalization}} in {{Deep Learning}}},
  author = {Neyshabur, Behnam and Bhojanapalli, Srinadh and McAllester, David and Srebro, Nathan},
  year = {2017},
  month = jul,
  number = {arXiv:1706.08947},
  eprint = {1706.08947},
  primaryclass = {cs},
  publisher = {arXiv},
  doi = {10.48550/arXiv.1706.08947},
  urldate = {2024-07-11},
  archiveprefix = {arXiv}
}

@misc{ollivier_ricci_2007,
  title = {Ricci Curvature of {{Markov}} Chains on Metric Spaces},
  author = {Ollivier, Yann},
  year = {2007},
  month = jul,
  number = {arXiv:math/0701886},
  eprint = {math/0701886},
  publisher = {arXiv},
  doi = {10.48550/arXiv.math/0701886},
  urldate = {2024-07-04},
  archiveprefix = {arXiv}
}

@misc{polyanskiy_wasserstein_2016,
  title = {Wasserstein Continuity of Entropy and Outer Bounds for Interference Channels},
  author = {Polyanskiy, Yury and Wu, Yihong},
  year = {2016},
  month = feb,
  number = {arXiv:1504.04419},
  eprint = {1504.04419},
  publisher = {arXiv},
  doi = {10.48550/arXiv.1504.04419},
  urldate = {2024-11-27},
  archiveprefix = {arXiv}
}

@misc{raginsky_non-convex_2017,
  title = {Non-Convex Learning via {{Stochastic Gradient Langevin Dynamics}}: A Nonasymptotic Analysis},
  shorttitle = {Non-Convex Learning via {{Stochastic Gradient Langevin Dynamics}}},
  author = {Raginsky, Maxim and Rakhlin, Alexander and Telgarsky, Matus},
  year = {2017},
  month = jun,
  number = {arXiv:1702.03849},
  eprint = {1702.03849},
  primaryclass = {cs, math, stat},
  publisher = {arXiv},
  doi = {10.48550/arXiv.1702.03849},
  urldate = {2023-06-06},
  archiveprefix = {arXiv}
}

@misc{rudolf_perturbation_2017,
  title = {Perturbation Theory for {{Markov}} Chains via {{Wasserstein}} Distance},
  author = {Rudolf, Daniel and Schweizer, Nikolaus},
  year = {2017}
}

@misc{schilling_introduction_2016,
  title = {An {{Introduction}} to {{L{\'e}vy}} and {{Feller Processes}}. {{Advanced Courses}} in {{Mathematics}} - {{CRM Barcelona}} 2014},
  author = {Schilling, Ren{\'e} L.},
  year = {2016},
  month = oct,
  number = {arXiv:1603.00251},
  eprint = {1603.00251},
  primaryclass = {math},
  publisher = {arXiv},
  doi = {10.48550/arXiv.1603.00251},
  urldate = {2022-10-10},
  archiveprefix = {arXiv}
}

@article{seeger_pac-bayesian_2002,
  title = {{{PAC-Bayesian Generalisation Error Bounds}} for {{Gaussian Process Classification}}},
  author = {{Seeger}, Matthias},
  year = {2002},
  journal = {Journal of Machine Learning Research}
}

@article{teugels_note_1972,
  title = {A {{Note}} on {{Poisson-Subordination}}},
  author = {Teugels, Jozef L.},
  year = {1972},
  month = apr,
  journal = {The Annals of Mathematical Statistics},
  volume = {43},
  number = {2},
  pages = {676--680},
  publisher = {Institute of Mathematical Statistics}
}

@misc{vallee_depoissonisation_2018,
  title = {The {{Depoissonisation}} Quintet: {{Rice-Poisson-Mellin-Newton-Laplace}}},
  shorttitle = {The {{Depoissonisation}} Quintet},
  author = {Vall{\'e}e, Brigitte},
  year = {2018},
  month = feb,
  number = {arXiv:1802.04988},
  eprint = {1802.04988},
  primaryclass = {math},
  publisher = {arXiv},
  doi = {10.48550/arXiv.1802.04988},
  urldate = {2024-08-22},
}

@misc{harel2025temperatureneedgeneralizationlangevin,
      title={Temperature is All You Need for Generalization in Langevin Dynamics and other Markov Processes}, 
      author={Itamar Harel and Yonathan Wolanowsky and Gal Vardi and Nathan Srebro and Daniel Soudry},
      year={2025},
      eprint={2505.19087},
      archivePrefix={arXiv},
      primaryClass={cs.LG},
}

@article{Gronwall1919NoteOT,
  title={Note on the Derivatives with Respect to a Parameter of the Solutions of a System of Differential Equations},
  author={T. H. Gronwall},
  journal={Annals of Mathematics},
  year={1919},
  volume={20},
  pages={292}
}

@article{hyvarinen_estimation_2005,
  title = {Estimation of {{Non-Normalized Statistical Models}} by {{Score Matching}}},
  author = {Hyv{\"a}rinen, Aapo},
  year = {2005},
  journal = {Journal of Machine Learning Research},
  volume = {6},
  number = {24},
  pages = {695--709},
  issn = {1533-7928},
  urldate = {2026-02-02}
}

@misc{dupuis2025renyidifferentialprivacyheavytailed,
      title={R\'enyi Differential Privacy for Heavy-Tailed {SDEs} via Fractional {Poincar\'e} Inequalities}, 
      author={Benjamin Dupuis and Mert Gürbüzbalaban and Umut Şimşekli and Jian Wang and Sinan Yildirim and Lingjiong Zhu},
      year={2025},
      eprint={2511.15634},
      archivePrefix={arXiv}
}

@article{van_erven_renyi_2014,
  title = {R{\'e}nyi {{Divergence}} and {{Kullback-Leibler Divergence}}},
  author = {{van Erven}, Tim and Harremo{\"e}s, Peter},
  year = {2014},
  month = jul,
  journal = {IEEE Transactions on Information Theory},
  volume = {60},
  number = {7},
  eprint = {1206.2459},
  primaryclass = {cs, math, stat},
  pages = {3797--3820}
}

@book{vapnik2000learning,
  author       = {Vladimir Naumovich Vapnik},
  title        = {{The Nature of Statistical Learning Theory, Second Edition}},
  series       = {Statistics for Engineering and Information Science},
  publisher    = {Springer},
  year         = {2000},
}

@InProceedings{orvieto_explicit_2023,
  title = 	 {Explicit Regularization in Overparametrized Models via Noise Injection},
  author =       {Orvieto, Antonio and Raj, Anant and Kersting, Hans and Bach, Francis},
  booktitle = 	 {Proceedings of The 26th International Conference on Artificial Intelligence and Statistics},
  pages = 	 {7265--7287},
  year = 	 {2023},
  editor = 	 {Ruiz, Francisco and Dy, Jennifer and van de Meent, Jan-Willem},
  volume = 	 {206},
  series = 	 {Proceedings of Machine Learning Research},
  month = 	 {25--27 Apr},
  publisher =    {PMLR},
}

@inproceedings{liu_noisy_2021,
  title = {Noisy {{Gradient Descent Converges}} to {{Flat Minima}} for {{Nonconvex Matrix Factorization}}},
  booktitle = {Proceedings of {{The}} 24th {{International Conference}} on {{Artificial Intelligence}} and {{Statistics}}},
  author = {Liu, Tianyi and Li, Yan and Wei, Song and Zhou, Enlu and Zhao, Tuo},
  year = {2021},
  month = mar,
  pages = {1891--1899},
  publisher = {PMLR}
}

@article{nesterov_smooth_2005,
  title = {Smooth Minimization of Non-Smooth Functions},
  author = {Nesterov, {\relax Yu}.},
  year = {2005},
  month = may,
  journal = {Mathematical Programming},
  volume = {103},
  number = {1},
  pages = {127--152}
}

@inproceedings{
cohen2021gradient,
title={Gradient Descent on Neural Networks Typically Occurs at the Edge of Stability},
author={Jeremy Cohen and Simran Kaur and Yuanzhi Li and J Zico Kolter and Ameet Talwalkar},
booktitle={International Conference on Learning Representations},
year={2021}
}

@misc{orvieto_anticorrelated_2023,
  title = {Anticorrelated {{Noise Injection}} for {{Improved Generalization}}},
  author = {Orvieto, Antonio and Kersting, Hans and Proske, Frank and Bach, Francis and Lucchi, Aurelien},
  year = {2023},
  month = may,
  number = {arXiv:2202.02831},
  eprint = {2202.02831},
  primaryclass = {stat},
  publisher = {arXiv},
  doi = {10.48550/arXiv.2202.02831},
  urldate = {2025-08-20},
  archiveprefix = {arXiv}
}

@book{villani_optimal_2009,
  title = {Optimal Transport - {{Old}} and {{New}}},
  author = {Villani, C{\'e}dric},
  year = {2009},
  publisher = {Springer},
}

@article{wang_transport-information_2020,
  title = {Transport-Information Inequalities for {{Markov}} Chains},
  author = {Wang, Neng-Yi and Wu, Liming},
  year = {2020},
  month = jun,
  journal = {The Annals of Applied Probability},
  volume = {30},
  number = {3},
  pages = {1276--1320},
  publisher = {Institute of Mathematical Statistics}
}

@article{wu_new_2000,
  title = {A New Modified Logarithmic {{Sobolev}} Inequality for {{Poisson}} Point Processes and Several Applications},
  author = {Wu, Liming},
  year = {2000},
  journal = {Probability theory and related fields, 118, 427--438}
}

@article{xu_information-theoretic_2017,
  title = {Information-Theoretic Analysis of Generalization Capability of Learning Algorithms},
  author = {Xu, Aolin and Raginsky, Maxim},
  year = {2017},
  month = nov,
  journal = {Advances in Neural Information Processing Systems 30 (NIPS 2017)},
  eprint = {1705.07809},
  primaryclass = {cs, math, stat},
  urldate = {2022-10-27},
  archiveprefix = {arXiv}
}

@misc{simsekli_differential_2024,
  title = {Differential {{Privacy}} of {{Noisy}} ({{S}}){{GD}} under {{Heavy-Tailed Perturbations}}},
  author = {{\c S}im{\c s}ekli, Umut and G{\"u}rb{\"u}zbalaban, Mert and Y{\i}ld{\i}r{\i}m, Sinan and Zhu, Lingjiong},
  year = {2024},
  month = mar,
  number = {arXiv:2403.02051},
  eprint = {2403.02051},
  primaryclass = {cs, math, stat},
  publisher = {arXiv},
  doi = {10.48550/arXiv.2403.02051},
  urldate = {2024-03-07}
}

@article{li_stochastic_2018,
  title={Stochastic Modified Equations and Dynamics of Stochastic Gradient Algorithms I: Mathematical Foundations},
  author={Qianxiao Li and Cheng Tai and E Weinan},
  journal={Proceedings of the 34th International Conference on Machine Learning},
  year={2018}
}

@inproceedings{li_validity_2021,
 author = {Li, Zhiyuan and Malladi, Sadhika and Arora, Sanjeev},
 booktitle = {Advances in Neural Information Processing Systems},
 editor = {M. Ranzato and A. Beygelzimer and Y. Dauphin and P.S. Liang and J. Wortman Vaughan},
 pages = {12712--12725},
 publisher = {Curran Associates, Inc.},
 title = {On the Validity of Modeling SGD with Stochastic Differential Equations (SDEs)},
 volume = {34},
 year = {2021}
}

@inproceedings{simsekli_tail-index_2019,
  title = {A {{Tail-Index Analysis}} of {{Stochastic Gradient Noise}} in {{Deep Neural Networks}}},
  booktitle = {Proceedings of the 36 Th {{International Conference}} on {{Machine Learning}} ({{ICML}} 2019)},
  author = {Simsekli, Umut and Sagun, Levent and Gurbuzbalaban, Mert},
  year = {2019},
  month = jan,
  eprint = {1901.06053},
  primaryclass = {cs, stat}
}

@InProceedings{cheng_stochastic_2020,
  title = 	 {Stochastic Gradient and {L}angevin Processes},
  author =       {Cheng, Xiang and Yin, Dong and Bartlett, Peter and Jordan, Michael},
  booktitle = 	 {Proceedings of the 37th International Conference on Machine Learning},
  pages = 	 {1810--1819},
  year = 	 {2020},
  editor = 	 {III, Hal Daumé and Singh, Aarti},
  volume = 	 {119},
  series = 	 {Proceedings of Machine Learning Research},
  month = 	 {13--18 Jul},
  publisher =    {PMLR}
}

@misc{hodgkinson2022generalizationboundsusinglower,
      title={Generalization Bounds using Lower Tail Exponents in Stochastic Optimizers}, 
      author={Liam Hodgkinson and Umut Şimşekli and Rajiv Khanna and Michael W. Mahoney},
      year={2022},
      eprint={2108.00781},
      archivePrefix={arXiv},
      primaryClass={stat.ML}
}

@inproceedings{raj2023algorithmic2,
  title={Algorithmic stability of heavy-tailed stochastic gradient descent on least squares},
  author={Raj, Anant and Barsbey, Melih and Gurbuzbalaban, Mert and Zhu, Lingjiong and {\c{S}}im, Umut and others},
  booktitle={International Conference on Algorithmic Learning Theory},
  pages={1292--1342},
  year={2023},
  organization={PMLR}
}

@misc{zhu_uniform--time_2023-1,
  title = {Uniform-in-{{Time Wasserstein Stability Bounds}} for ({{Noisy}}) {{Stochastic Gradient Descent}}},
  author = {Zhu, Lingjiong and Gurbuzbalaban, Mert and Raj, Anant and Simsekli, Umut},
  year = {2023},
  month = oct,
  number = {arXiv:2305.12056},
  eprint = {2305.12056},
  primaryclass = {cs, math, stat},
  publisher = {arXiv},
  doi = {10.48550/arXiv.2305.12056},
  urldate = {2024-07-04},
  archiveprefix = {arXiv}
}

@book{amari2016information,
  title={Information geometry and its applications},
  author={Amari, Shun-ichi},
  volume={194},
  year={2016},
  publisher={Springer}
}

\end{document}